\newif\ifPREPRINT
\PREPRINTtrue 	%

\documentclass[journal]{IEEEtran}
\ifCLASSINFOpdf
\else
\fi
\usepackage{amsmath} %
\usepackage{amssymb}  %
\usepackage{amsthm}
\usepackage{amsfonts}
\usepackage{mathtools}

\usepackage{bm}
\providecommand{\bm}{\pmb}

\usepackage{graphicx}
\usepackage{caption}
\usepackage{subcaption}

\graphicspath{{./Images/}%
			  {./Images/Model/}%
			  {./Images/Controller/}}%

\usepackage{paralist}

\usepackage[bookmarks]{hyperref} %
\hypersetup{
    colorlinks,
    citecolor=black,
    filecolor=black,
    linkcolor=black,
    urlcolor=black,
    pdfauthor={},
    pdfsubject={},
    pdftitle={}
}

\usepackage[usenames,dvipsnames,table]{xcolor}

\usepackage{cite}

\usepackage{color}

\newcommand{\vect}[1]{\bm{#1}}		%
\newcommand{\matr}[1]{\bm{#1}}		%
\newcommand{\nR}[1]{\mathbb{R}^{#1}}		%
\newcommand{\matrice}[1]{\begin{bmatrix} #1 \end{bmatrix}}	%
\newcommand{\upperRomannumeral}[1]{\uppercase\expandafter{\romannumeral#1}}	%

\newcommand{\vSpace}{\;\,}

\newcommand{\diag}[1]{\text{diag}\left( #1 \right)}

\DeclarePairedDelimiter{\norm}{\lVert}{\rVert} %
\DeclareMathOperator*{\argmin}{arg\,min}

\newcommand{\fig}{Fig.~}	%
\newcommand{\sect}{Sect.~}	%
\newcommand{\theo}{Theorem~}	%

\newcommand{\GenFrame}{\mathcal{F}}		%
\newcommand{\origin}{O}						%
\newcommand{\vX}{\vect{x}}					%
\newcommand{\vY}{\vect{y}}					%
\newcommand{\vZ}{\vect{z}}					%
\newcommand{\pos}{\vect{p}}				%
\newcommand{\dpos}{\vect{v}} %
\newcommand{\ddpos}{\dot{\dpos}} %
\newcommand{\vZero}{\vect{0}}				%
\newcommand{\eye}[1]{\matr{I}_{#1}}		%
\newcommand{\frameW}{\GenFrame_W}			%
\newcommand{\originW}{\origin_W}		%
\newcommand{\xW}{\vX_W}				%
\newcommand{\yW}{\vY_W}				%
\newcommand{\zW}{\vZ_W}				%

\newcommand{\frameH}{\GenFrame_H}			%
\newcommand{\originH}{\origin_H}			%
\newcommand{\xH}{\vX_H}				%
\newcommand{\yH}{\vY_H}				%
\newcommand{\zH}{\vZ_H}				%
\newcommand{\pH}{\pos_H}			%
\newcommand{\pHx}{p_{Hx}}			%
\newcommand{\pHy}{p_{Hy}}			%
\newcommand{\dpH}{\dpos_H}			%
\newcommand{\dpHX}{v_{Hx}}			%
\newcommand{\dpHY}{v_{Hy}}			%
\newcommand{\dpHZ}{v_{Hz}}			%
\newcommand{\ddpH}{\ddpos_H}		%
\newcommand{\massH}{m_H}				%
\newcommand{\dampingH}{\matr{B}_H}	%
\newcommand{\groundForceIntensity}{f_g}
\newcommand{\groundForce}{ \vect{f}_g}

\newcommand{\pHRef}{{\pH^r}}

\newcommand{\pHX}{p_{Hx}}

\newcommand{\pHY}{p_{Hy}}

\newcommand{\pHZ}{p_{Hz}}
\newcommand{\pHDes}{{\pH^d}}
\newcommand{\humanForce}{\vect{u}_H}
\newcommand{\gravityH}{\vect{g}_H}

\newcommand{\length}{\bar{l}_{c}}					%
\newcommand{\springCoeff}{k_{c}}			%
\newcommand{\cableForce}{\vect{f}_{c}}			%
\newcommand{\cableForceEq}{\bar{\vect{f}}_{c}}			%
\newcommand{\cableForceX}{{f}_{cx}}			%
\newcommand{\cableForceY}{{f}_{cy}}			%
\newcommand{\cableForceZ}{{f}_{cz}}			%
\newcommand{\cableAttitude}{\vect{l}_{c}}%
\newcommand{\cableAttitudeNorm}{\norm{\vect{l}_{c}}}%
\newcommand{\cableAttitudeNormOf}[1]{\norm{\vect{l}_{c}(#1)}}%
\newcommand{\dCableAttitude}{\dot{\vect{l}}_{c}}
\newcommand{\tension}{t_{c}(\norm{\cableAttitude})}				%
\newcommand{\tensionOf}[1]{t_{c}(#1)}				%
\newcommand{\cableAttitudeNormInitial}{{l}_{c0}}%

\newcommand{\frameR}{\GenFrame_{R}}			%
\newcommand{\originR}{O_{R}}					%
\newcommand{\xR}{\vX_{R }}								%
\newcommand{\yR}{\vY_{R }}								%
\newcommand{\zR}{\vZ_{R }}								%
\newcommand{\pR}{\pos_{R }}						%
\newcommand{\pRX}{p_{Rx}}						%
\newcommand{\pRY}{p_{Ry}}						%
\newcommand{\pRZ}{p_{Rz}}						%
\newcommand{\dpR}{\dpos_{R }}					%
\newcommand{\dpRX}{v_{Rx}}						%
\newcommand{\dpRY}{v_{Ry}}						%
\newcommand{\dpRZ}{v_{Rz}}						%
\newcommand{\ddpR}{\ddpos_{R }}				%
\newcommand{\ddpRX}{\dot{v}_{Rx}}				%
\newcommand{\ddpRZ}{\dot{v}_{Rz}}				%
\newcommand{\uR}{\vect{u}_{R }}				%
\newcommand{\dampingA}{\matr{B}_{A}}		%
\newcommand{\inertiaA}{\matr{M}_{A}}		%
\newcommand{\uA}{\vect{u}_{A}}					%
\newcommand{\state}{\vect{x}}						%
\newcommand{\stateSet}{\mathcal{X}}						%
\newcommand{\stateSetTaut}{\stateSet_{T}}						%
\newcommand{\stateSetSlack}{\stateSet_{S}}						%
\newcommand{\stateEq}{\bar{\vect{x}}}						%
\newcommand{\dstate}{\dot{\vect{x}}}						%

\newcommand{\dynamics}[1]{\vect{f}(#1)}

\newcommand{\kpH}{k_{H}}	%
\newcommand{\KpH}{\matr{K}_{H}}	%
\newcommand{\cableForceZDes}{{f}_{z}}			%
\newcommand{\cableForceRef}{{\cableForce^r}}		%
\newcommand{\cableForceDes}{{\vect{f}^d}}		%
\newcommand{\cableForceDesXY}{{\vect{f}^d_{xy}}}		%
\newcommand{\cableForceDesXYNorm}{{{f}^d_{xy}}}		%
\newcommand{\cableForceDesZ}{{\vect{f}^d_z}}		%
\newcommand{\cableForceRefFinal}{{\vect{f}^r_T}}		%
\newcommand{\errorR}{\vect{e}_R}
\newcommand{\errorRFinal}{\vect{e}_{RT}}
\newcommand{\errorRX}{{e}_{Rx}}
\newcommand{\errorRY}{{e}_{Ry}}
\newcommand{\pRRef}{{\pR^r}}
\newcommand{\pRXRef}{{p_{Rx}^r}}
\newcommand{\pRYRef}{{p_{Ry}^r}}

\newcommand{\pRDes}{{\pR^d}}
\newcommand{\timeVaryingInput}{\uA'}

\newcommand{\passiveInput}{\vect{u}}
\newcommand{\passiveOutput}{\vect{y}}
\newcommand{\passiveInputTrajectory}{\vect{u}_1}
\newcommand{\passiveOutputTrajectory}{\vect{y}_1}
\newcommand{\passiveInputHumanForce}{\vect{u}_2}
\newcommand{\passiveOutputHumanForce}{\vect{y}_2}

\newcommand{\lyapunovFunction}{V(\state)} %
\newcommand{\dLyapunovFunction}{\dot{V}(\state)} %

\newcommand{\invariantSet}{\Omega_{\alpha}}		%
\newcommand{\invariantSetZero}{\Omega_{0}}		%
\newcommand{\maxInvariantSet}{\mathcal{M}}
\newcommand{\dVZeroSet}{\mathcal{E}}				%

\newcommand{\param}{s} 
\newcommand{\paramOpt}{s^\star} 

\newcommand{\pHDesPerp}{\pHDes^\perp}

\newcommand{\errorH}{\vect{e}_H}
\newcommand{\errorHPath}{{e}_H^P}
\newcommand{\errorHPathMean}{\bar{{e}}^P_H}
\newcommand{\errorHPathStd}{{\sigma}_H}

\newtheorem{thm}{Theorem}

\theoremstyle{definition}

\theoremstyle{remark}
\newtheorem{rmk}{Remark}

\begin{document}

\newcommand{\titlePaper}{Physical Human-Robot Interaction with a Tethered Aerial Vehicle: Application to a Force-based Human Guiding Problem}
\newcommand{\authorsPaper}{Marco~Tognon~\IEEEmembership{Member,~IEEE,} Rachid~Alami~\IEEEmembership{Member,~IEEE,} Bruno~Siciliano~\IEEEmembership{Fellow,~IEEE}}

\title{\titlePaper}

\author{\authorsPaper%
\thanks{M. Tognon is with the Autonomous Systems Lab, Department of Mechanical and Process Engineering, ETH Zurich, 8092 Z\"urich, Switzerland, {\tt \scriptsize\href{mailto:mtognon@ethz.ch}{mtognon@ethz.ch}}. The work has been done while the author was at LAAS-CNRS, Toulouse, France.}%
\thanks{R. Alami is with LAAS-CNRS, Universit\'e de Toulouse, CNRS, Toulouse, France, {\tt \scriptsize \href{mailto:rachid.alami@laas.fr}{rachid.alami@laas.fr}.}}%
\thanks{B. Siciliano is with the Department of Electrical Engineering and Information Technology
University of Naples Federico II
Via Claudio 21, 80125 Naples, Italy  {\tt \scriptsize \href{mailto:bruno.siciliano@unina.it}{bruno.siciliano@unina.it}.}}

\ifPREPRINT
	{}
\else	
	\thanks{Manuscript received Month DD, YYYY; revised Month DD, YYYY.}
\fi
}

\ifPREPRINT
	{}
\else	
\markboth{IEEE TRANSACTIONS ON ROBOTICS,~Vol.~XX, No.~XX, Month~YEAR}%
{Tognon \MakeLowercase{\textit{et al.}}: \titlePaper}
\fi
\maketitle

\begin{abstract}
	Today, physical Human-Robot Interaction (pHRI) is a very popular topic
in the field of ground manipulation. At the same time, Aerial Physical
Interaction (APhI) is also developing very fast. Nevertheless, pHRI with
aerial vehicles has not been addressed so far. In this work, we present
the study of one of the first systems in which a human is physically
connected to an aerial vehicle by a cable. We want the robot to be
able to pull the human toward a desired position (or along a path) only
using forces as an indirect communication-channel.
We propose an admittance-based approach that makes pHRI safe. A controller, inspired by the literature on flexible manipulators, computes the desired interaction forces that properly guide the human. The stability of the system is formally proved with a Lyapunov-based argument. The system is also shown to be passive, and thus robust to non-idealities like additional human forces, time-varying inputs, and other
external disturbances. We also design a maneuver regulation policy to
simplify the path following problem. The global method has been experimentally validated on a group of four subjects, showing a reliable and safe pHRI.
\end{abstract}
\IEEEpeerreviewmaketitle

\section{Introduction}\label{sec:intro}

\IEEEPARstart{A}{erial} robotics is one of the research fields receiving a constantly growing interest.
This is motivated by the very wide range of applications which \textit{Unmanned Aerial Vehicles} (UAVs) could be useful for.
Popular examples are monitoring, surveillance, agriculture and many others. %
Beyond these contact-free applications, many investments and efforts have been recently dedicated to \textit{Aerial Physical Interaction} (APhI).
Motivated by its scientific interest and great potential business,  several research labs and companies are conceiving and studying new aerial robots with manipulation capabilities that can interact with the environment~\cite{2018-RugLipOll}.

The research community already proposed several solutions to enhance physical interaction capabilities of aerial vehicles.
The most simple solution consists in using a rigid tool attached to a standard unidirectional-thrust aerial vehicle. This allows exchanging forces with the environment, e.g., by pushing or sliding~\cite{2013-NguLee,2016-BarCapHamStrFum}.
In order to cope with the underactuation of the system, aerial vehicles equipped with single or multiple articulated arms (also called \textit{aerial manipulators}) have been proposed~\cite{2017g-TogYueBuoFra,2019-CatReaSuaDilPieAntCacHerOll}.
Furthermore, the recent design of tilted or tiltable multi-rotor platforms that are fully-actuated, allowed improving aerial physical interaction capabilities even further~\cite{2016j-RylBicFra,kamel2018voliro}.
Thanks to the several proposed platforms and control methods, the feasibility of physical interaction with aerial robots has been shown for different tasks like pushing and sliding~\cite{2019e-TogTelGasSabBicMalLanSanRevCorFra}, transportation~\cite{2011-BerKonMazOll}, etc.
This opens the door to new application domains like contact-based inspection~\cite{2018m-OllHerFraAntKonSanVigSanTruBalAndRod}, assembly and decommissioning.

With the advance of aerial robotics, the presence of such robots in our daily life will grow as well. 
This means that the next generation of aerial robots must be able to safely and reliably interact with the environment, but also with humans.
So far, \textit{Human-Robot Interaction} (HRI) has been rarely addressed in the aerial robotics community (likely due to safety concerns).
The majority of the works focus on remote control.
Bilateral teleoperation methods have been presented to help humans control single and multiple aerial vehicles navigating in cluttered environments~\cite{2014-MerStrCar}.
Only recently, the teleoperation problem has been extended to APhI as well~\cite{2015c-GioMohFraPra}.  

Nevertheless, in all these cases, humans and robots are in separated environments and the interaction is only virtual, using haptic devices.
The study of scenarios where aerial robots share the same environment of humans is still poorly addressed and limited to contactless HRI. 
In~\cite{2017-PesHitKau}, the authors propose a survey on the use of natural communication means like gestures, speech and gaze direction to let the human interface with the robot.

Beyond the study of interface methods between humans and aerial robots, another important topic recently explored is the social acceptance of such robots~\cite{2017-AchBevDun}.
In~\cite{2013-LieYai} the problem is addressed from a design point of view, analyzing how the propellers of multi-rotor platforms affect the comfort of humans.
In a different direction, a small number of papers such as \cite{2019-WojFreSasShaCau,2013-DunMur}  have recently started to investigate the social constraints (distance, speed, direction of approach, etc.) that should be applied to the motion of an aerial robot in the proximity of a human. 
While human-aware motion planning issues have been substantially studied for ground robots\cite{2013-KruPanAlaKir, 2015-RioSpaLau}, it is still very preliminary for aerial robots. This is also true for cognitive interaction \cite{2016-ThoHofCca, 2017-LemWarSisCloAla}.

Therefore, the works presented so far focus only on contactless HRI for  aerial robots.
On the other hand, \textit{physical Human-Robot Interaction} (pHRI) for aerial vehicles is still an unexplored domain.
To the best of our knowledge, the only work addressing this topic is~\cite{2013-AugDan} where a compliant controller based on an admittance filter is proposed. 
Compliance with respect to external forces is shown, but the interaction with the human is limited to simple contacts.

Differently, research done on pHRI with ground manipulators is already quite extensive~\cite{2008-DesSicDelBic,2018-AjoZanIvaAlbKosKha}.
Many control methods, most of them based on admittance~\cite{2015-FicVilSic,2007-AlbHadOttSteWimHir,1985-Hog} and passivity-based methods~\cite{1988-ColHog}, have been proposed for applications like co-manipulation of heavy loads~\cite{2014-AgrCheBusGerKhe}. 

To advance the study of pHRI methods for aerial robots, in this work we face one of the first systems in which a human and an aerial robot are tightly coupled by a physical means.
In particular, we consider a human holding a handle which is in turn connected to an aerial vehicle by a cable (see \fig\ref{fig:model}).
In this manuscript, we want to design a control strategy that allows the robot to safely ``physically guide'' the human toward a desired position (or along a predefined path) exploiting the cable as a force-based communication means.
In previous works~\cite{2017a-TogFra,2016c-TogDasFra}, we have already shown that \textit{tethered aerial vehicles} are very suitable for physical interaction and that the forces along a cable can be successfully used as an indirect communication means among robotic agents~\cite{2018h-TogGabPalFra}.

The problem of guiding humans by aerial vehicles has indeed been already addressed in the literature but proposing contactless approaches only, e.g., based on sound~\cite{2015-AviFunHen} or vision~\cite{2015-MueMui}.
In this work, we rather exploit the physical interaction using forces as more direct feedback.  
This solution particularly suits for persons with visual or hearing limitations.
There have been a number of contributions of ground robots guiding humans. To cite a few: \cite{BurgardCFHLSST99, SiegwartABBFGJLMMPPRTT03, ClodicFACBBCDDE06, TriebelAABBCCCE15}. However, while they have considered a variety of capabilities to take into account human constraints and a variety of interaction modalities, only a small number of contributions considered pulling mobile robots or walking helpers such as autonomous walkers for assisting elderly  \cite{Morris2003} or visually impiared people \cite{Tachi1985,2009-MelPraNagIll}.
On the other hand, the use of aerial vehicles in physical interaction with humans demands a careful treatment of the interaction forces and the stability of the system.

Inspired by the extensive literature of pHRI for ground manipulators, we firstly propose the use of an \textit{admittance-based strategy} to make the robot comply with the human. 
We shall show that exploiting this approach we can impose a desired cable force that has the function to pull the human toward the desired point.
The design of such input requires particular attention due to the flexible nature of the physical link connecting the human to the robot.
Taking inspiration from the theory on control of manipulators with elastic joints~\cite{2008-DelBoo_cha13}, we design a \textit{robot-feedback solution} together with a feedforward term. 
The latter enforces the tautness of the cable.
Approximating the human behavior as a mass-spring-damper system and considering the overall closed-loop system, we formally prove (though Lyapunov theory) that the proposed control method allows the robot to asymptotically steer the human to the desired position.

We also show that the system is \textit{output-strictly passive}. 
This guarantees the stability of the system even in case of
\begin{inparaenum}[i)]
	\item model errors, 
	\item time-varying reference for the human position, and 
	\item additional forces applied by the human to the cable (e.g. if he/she wants to stop). 
\end{inparaenum}
This property grants robustness and safety of the method.

To address the problem of path following, we design a \textit{maneuver regulation method} based on the works in \cite{2013n-SpeNotBueFra,1995-HauHin}.
This avoids the excessive growth of the tracking error which may lead to input saturation and thus instability, or to big deviation from the desired path.

The control performance and stability of the system are validated by real experiments firstly showing (from the best of our knowledge) safe and reliable physical interaction between a human and an aerial vehicle.   
 
The manuscript is organized as follows:
In \sect\ref{sec:modeling} we model the entire system. 
In \sect\ref{sec:control} and \sect\ref{sec:equilibriaAndStability}, we describe the proposed control strategy and we analyze the stability of the closed-loop system.
The additional passivity property is showed in \sect\ref{sec:passivity}. 
In \sect\ref{sec:pathFollowing}, we describe the path following approach.
The experimental results validating the proposed method are reported in \sect\ref{sec:experiments}.
Section \ref{sec:conclusions} ends the manuscript with some final discussion.

\section{Modeling}\label{sec:modeling}
\begin{figure}
	\centering
	\includegraphics[width = \columnwidth]{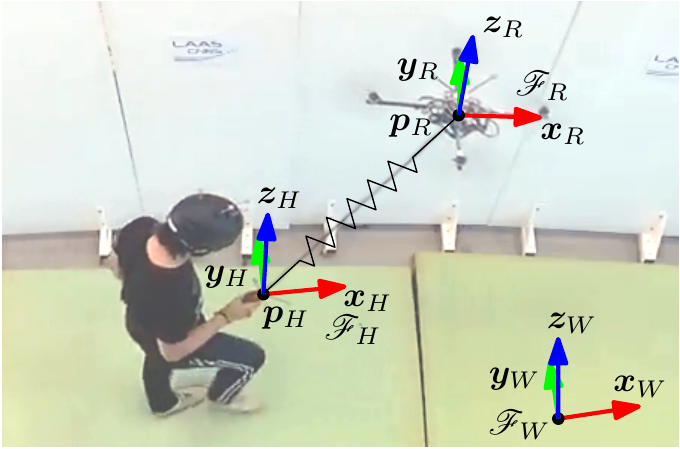}
	\caption{Representation of the aerial human-tethered guiding system.}
	\label{fig:model}
\end{figure}

The system here analyzed (see \fig\ref{fig:model}) is composed of an aerial vehicle tethered by a cable to a handle held by the hand of a human being. 
To describe it, we first define an \textit{inertial frame}, $\frameW = \{\originW,\xW,\yW,\zW\}$, where $\originW$ is the arbitrary origin and $\{\xW,\yW,\zW\}$ are its unit axes.  %
In particular, $\zW$ is oriented in the opposite direction to the gravity vector. 

We consider the position of the human as the position of his/her hand holding the handle. 
We define a \textit{body frame} rigidly attached to the handle,  $\frameH = \{\originH,\xH,\yH,\zH\}$.
$\originH$ is the origin of $\frameH$ and is centered in the point where the cable is attached to the handle. 
$\{\xH,\yH,\zH\}$ are its unit axes.
The state of the human is then given by the position of $\originH$ and its linear velocity, defined by the vectors $\pH = [\pHX \vSpace \pHY \vSpace \pHZ]^\top \in \nR{3}$ and $\dpH = [\dpHX \vSpace \dpHY \vSpace \dpHZ]^\top \in \nR{3}$, respectively, both with respect to $\frameW$. 

For the sake of designing a human-friendly robot controller, it is important to consider the human's behavior when interacting with the environment.
In \cite{1994-IkeMonIno}, the authors analyzed the interaction forces during a task in which two humans cooperatively moves an object. 
It turned out that, for this kind of task, the human's dynamics can be approximated with a mass-spring-damper system.
The validity of such a model has been confirmed in other works like~\cite{2005-SpeShaGol}.
 In \cite{1998-IkeMiz,2012-MorLawKucSezBasHir,1997-KosKaz,1993-KosFujFuk,2007-DucGos}, this impedance model has been used as a basis to develop human-robot cooperative or teleoperated manipulation tasks.  
We consider the human's dynamics as:
\begin{align}
	\massH \ddpH = -\gravityH - \dampingH\dpH + \cableForce + \groundForce,
	\label{eqn:humanModel}
\end{align}
where $\massH \in \nR{}_{>0}$ and the positive definite matrix $\dampingH\in \nR{3 \times 3}_{>0}$ are the apparent mass and damping, respectively; $\cableForce = [\cableForceX \vSpace \cableForceY \vSpace \cableForceZ]^\top \in \nR{3}$ is the cable force applied to the human at $\originH$; $\gravityH = \massH g \zW$; $g \in \nR{}_{>0}$  is the gravitational constant.
Without loss of generality, we can assume the ground being flat and horizontal.
$\groundForce$ is the ground reaction force such that:
\begin{itemize}
	\item $\groundForce^\top \xW = \groundForce^\top \yW = 0$,
	\item $\groundForce^\top \zW > 0$, and
	\item $\ddpH^\top \zW = \dpH^\top \zW = 0$, i.e., the human is constrained on the ground.
\end{itemize}
It follows that $\groundForce = \groundForceIntensity\zW$, where $\groundForceIntensity \in \nR{}_{\geq0}$ is the intensity of the reaction force.
Notice that it has to be $\cableForce^\top \zW < \massH g$.
Practically, this is always the case for standard small/medium size aerial vehicles.

The \textit{human mass}, $\massH$, models the reactivity of the human to external forces. 
A low mass means a high willingness of following external forces. 
For a high mass, it is the opposite.
The \textit{damping effect} models the willingness of the human to stay still.
Finally, in a mass-spring-damper system, the spring models the desire of reaching a specific point. 
However, in \eqref{eqn:humanModel} there is no spring term since, in our scenario, the human is not aware of the desired path. 
The human simply follows the external force applied by the robot through the cable.

\begin{rmk}
	The precise modeling of the human body dynamics is out of the scope of this work. 
	Through model \eqref{eqn:humanModel}, we want to capture only the macro navigation behavior of a human when pulled by an external force.
\end{rmk}

\begin{rmk}
	In this work, we focus on the control of the human's hand position rather than his/her head position. 
	This is motivated by the fact that a handle integrating positioning sensors like GPS, IMU, and cameras, is much more comfortable and wearable with respect to a fully equipped helmet. 
	Nevertheless, the use of a kinematic model could be employed to estimate the human's head position if required. 
\end{rmk}

The cable is attached from one side to the handle, and to the other side to the aerial vehicle at point $\originR$, which coincides with its center of mass.
To describe the aerial robot state we define the \textit{body frame} $\frameR = \{ \originR,\xR,\yR,\zR\}$, rigidly attached to it.
The robot configuration is then described by the position of $\originR$ with respect to (w.r.t.) $\frameW$ denoted by the vector $\pR = [\pRX \vSpace \pRY \vSpace \pRZ]^\top \in \nR{3}$.
We complete the state of the robot with the linear velocity $\dpR =  [\dpRX \vSpace \dpRY \vSpace \dpRZ]^\top \in \nR{3}$.
We neglect the rotational coordinates and relative velocity since it will be not relevant for the following derivations.

We assume that the robot is controlled by a position controller able to track any $C^2$  trajectory with negligible error in the domain of interest, independently from external disturbances. 
With the recent robust controllers (as the one in~\cite{2016j-RylBicFra,2020d-HamTogFra} for both unidirectional- and multidirectional-thrust vehicles) and disturbance observers for aerial vehicles, one can obtain very precise tracking, even in the presence of external disturbances. 
Nevertheless, in practice, precise tracking is not needed for stability, but only to achieve perfect performance.

The closed-loop translational dynamics of the robot subject to the position controller is assumed as the one of a double integrator: 
\begin{align}
	\ddpR = \uR,
	\label{eqn:robotModel}
\end{align}
where $\uR \in \nR{3}$ is a virtual input to be designed.
If we consider a multidirectional-thrust platform capable of controlling both position and orientation independently~\cite{2017e-RylMusPieCatAntCacFra}, the double integrator is an exact model of the closed-loop system apart from modeling errors. 
In the case of an underactuated unidirectional-thrust vehicle, the double integrator is instead a very good approximation in the working conditions, namely not aggressive maneuvers. 
Indeed, the rotational dynamics is totally decoupled from the translational one and it is much faster than the latter, allowing the application of the time-scale separation principle.
Furthermore, it is well known that a unidirectional-thrust vehicle is differentially flat w.r.t. to $\pR$ and the rotation along the thrust direction. 
According to \eqref{eqn:robotModel}, the platform might seem `infinitely stiff' with respect to interaction forces, like the one produced by the cable.
However, we shall suitably design the input $\uR$ so as to re-introduce a compliant behavior. 

Afterward, we model the cable that connects the aerial vehicle and the handle, or equivalently the human hand, as a unilateral spring along its principal direction.
We define the cable force produced at $\originH$ on the handle as:
\begin{align}\label{eqn:cableForce}
	\cableForce = \tension {\cableAttitude}/{\cableAttitudeNorm},
\end{align}
where $\tension$ represents the cable internal force intensity and $\cableAttitude = \pR - \pH$. 
The force produced on the other end of the cable, namely on the robot at $\originR$, is equal to $-\cableForce$.

In practice, the forces involved shall be very small (few newtons) and the distance between the human and the robot can be of few meters only.
For these reasons, a thin and short cable can be used which has negligible mass and inertia w.r.t. the ones of the robot and of the human. 

Defining by $\length \in \nR{}_{> 0}$ the constant nominal length of the cable, we say that the cable is:
\begin{enumerate}
	\item \textit{Slack} when the distance between the two ends is shorter than $\length$. In this case, the cable force intensity is zero;
	\item \textit{Taut} when the distance between the two ends is equal or greater than $\length$. In this case, the cable force intensity follows Hooke's law.
\end{enumerate} 
In particular, $\tension$, represented in \fig\ref{fig:model:cable_internal_force}, is defined as:
\begin{align}\label{eqn:cableForceIntensity}
	\begin{split}
	\tension = \begin{cases}
					\springCoeff(\cableAttitudeNorm - \length)& \text{ if } \cableAttitudeNorm - \length > 0 \\
					\vZero	&	\text{ otherwise }
				\end{cases},
	\end{split}			
\end{align}
where $\springCoeff \in \nR{}_{> 0}$ is the constant elastic coefficient.
\begin{figure}
	\centering
	\includegraphics[width = \columnwidth]{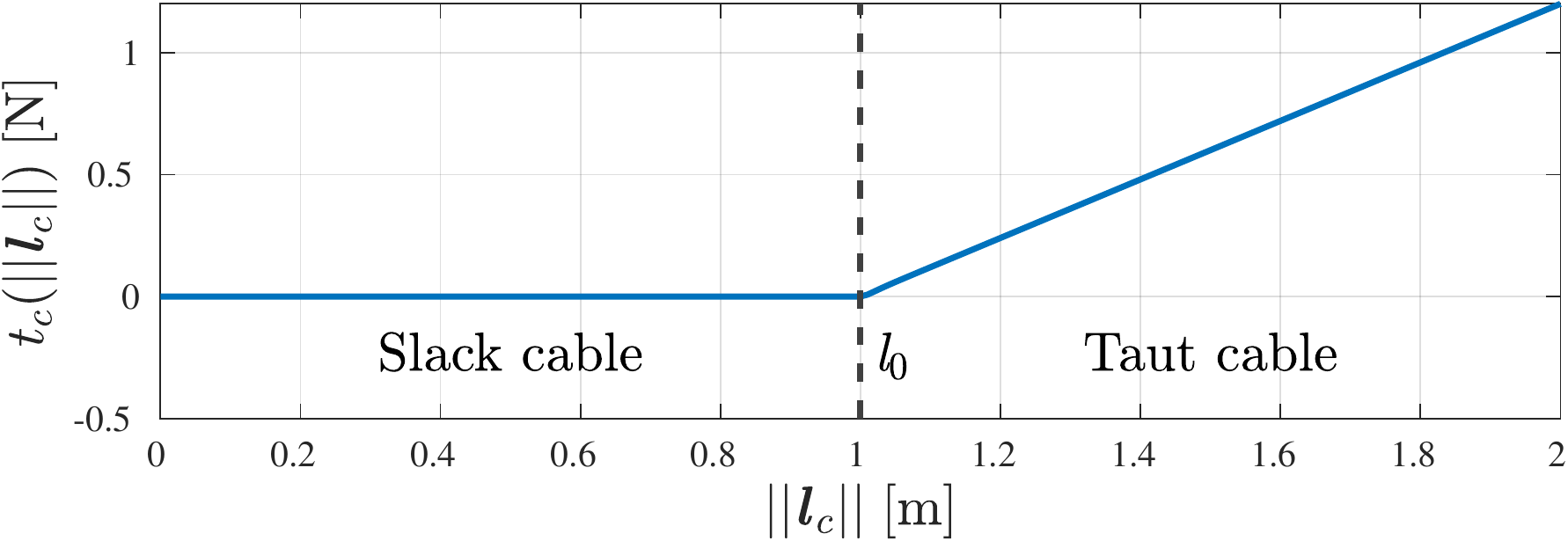}
	\caption{Representation of $\tension$ as in \eqref{eqn:cableForceIntensity}.}
	\label{fig:model:cable_internal_force}
\end{figure}

\begin{rmk}
Notice that $\tension$ can be any continuous and differentiable monotonically increasing function such that:
\begin{align}
	\tension \geq \epsilon \cableAttitudeNorm + \gamma, \text{ if } \cableAttitudeNorm - \length > 0, 
	\label{eqn:tensionMajority}
\end{align}
for some $\epsilon \in \nR{}_{> 0}$ and $\gamma \in \nR{}$. 
This allows modeling more complex behaviors of the cable during the initial stretching/unstretching. 
We also consider the system to be always far from the maximum feasible tension.
\end{rmk}

\section{Control Method}\label{sec:control}

Since the robot physically interacts with the human, it is important that the interaction is done in a compliant way, so as to always preserve stability and prevent the human to be armed or subjected to uncomfortable forces.
For this problem, and more in general for physical interaction problems, a very common solution consists of using an \textit{admittance control strategy}~\cite{2016-HadCro}.
According to this method, the desired motion of the robot is computed as the one of an object with a specific admittance (mass, stiffness, and damping) subjected to the measured external force acting on the robot~\cite{2009-SicSciVilOri}. 
In this way one can reshape the apparent mechanical characteristics of the system, making it more or less stiff or damped.
We define the robot control input $\uR$ as: 
\begin{align}\label{eqn:admittance}
	\uR = \inertiaA^{-1}\left( -\dampingA\dpR - \cableForce + \uA \right),
\end{align}
where the two positive definite diagonal matrices $\inertiaA,  \dampingA \in \nR{3 \times 3}$ are the virtual inertia and damping of the robot.
$\uA \in \nR{3}$ is an additional input that will be defined in the following in order to achieve the desired control goal.
Notice that in order to implement \eqref{eqn:admittance}, only the state of the robot $(\pR,\dpR)$, and the force applied by the cable $\cableForce$ are needed. 
The first can be retrieved with standard onboard sensors, while the second can be directly measured with a force sensor or estimated by a model-based observer as done in~\cite{2017e-RylMusPieCatAntCacFra,2017-TagKamVerSieNie}. 

It is useful to define the system state vector as $\state = [\pH^\top \vSpace \dpH^\top  \vSpace \pR^\top  \vSpace \dpR^\top]^\top \in \stateSet \subset \nR{12}$ and to write the dynamics in a more compact state form as $\dstate = \dynamics{\state,\uA}$, where:
\begin{align}
	\dynamics{\state,\uA} &= \matrice{\dpH \\  \frac{1}{\massH}\left(-\gravityH - \dampingH\dpH + \cableForce +  \groundForce\right) \\ \dpR \\ \inertiaA^{-1}\left( -\dampingA\dpR - \cableForce + \uA \right) },
	\label{eqn:closedLoopSystemDyn}
\end{align}
and $\cableForce$ is computed as in \eqref{eqn:cableForce}. 

\begin{rmk}
The system has a \textit{hybrid nature} characterized by two different behaviors according to the state of the cable:
\begin{enumerate}
	\item When $\state \in \stateSetSlack = \{ \state \in \stateSet \; | \;  \norm{\cableAttitude} - \length \leq 0 \}$, the cable is slack and the cable force is zero, i.e., $\cableForce = \vZero$. In this condition, the robot and human systems are completely independent;
	\item When $\state \in \stateSetTaut = \{ \state \in \stateSet \; | \;  \norm{\cableAttitude} - \length > 0 \}$, the cable is taut and the cable force is non-zero, i.e., $\cableForce \neq \vZero$. In this condition, the robot and human systems interact with each other through the cable force.
\end{enumerate}
\end{rmk}
 
Looking at the system dynamics when the cable is taut, it appears to be a couple of mass-damper elements connected by a spring. 
This system is similar to an elastic manipulator where the actuator (the robot) is connected to the end-effector (the human) by an elastic element (the cable).
For this system, we want to control the position of the end-effector (human position) by the input force $\uA$ generated by the actuator (robot).
In view of this parallelism, we naturally took inspiration from the state-of-the-art on control of manipulators with elastic joints~\cite{2008-DelBoo_cha13}.

Without the elastic element, the most natural choice for the design of a linear controller is to apply a simple proportional-derivative (PD) feedback from the end-effector position and velocity.
However, considering a single link rotating on a horizontal plane and actuated with a motor through an elastic joint coupling, it has been proven that a feedback action entirely based on the end-effector position and velocity leads to instability, no matter the gain values~\cite{2008-DelBoo_cha13}. 

The most preferred and grounded solution for manipulators with elastic joints is to apply a feedback completely based on the motor variables.
If the proportional and derivative gains are strictly positive, then the closed-loop system will be asymptotically stable.
Notice that, given a desired end-effector position, the desired motor value is computed inverting the dynamic model whose parameters need to be known.

To implement this control law, we only need a proportional feedback w.r.t. the robot position. The damping action w.r.t. the robot velocity is already present in \eqref{eqn:admittance},  
Thus, we set  
\begin{align}
	\uA = \KpH \errorR + \cableForceRef,
	\label{eqn:controller}
\end{align}
where $\KpH = \diag{\kpH,\kpH,0}$ with $\kpH \in \nR{}_{> 0}$, $\errorR = \pRRef - \pR$ is the robot position error, $\pRRef$ is the robot position of reference, and $\cableForceRef \in \nR{3}$ is a constant forcing input. 
In practice, $\cableForceRef$ is the desired cable force.
Figure~\ref{fig:controller:workingPrinciple} explains the working principle of the controller, while Figure \ref{fig:controller:blockDiagram} shows its block diagram.

\begin{figure}[t]
	\centering
	\includegraphics[width = 0.8\columnwidth]{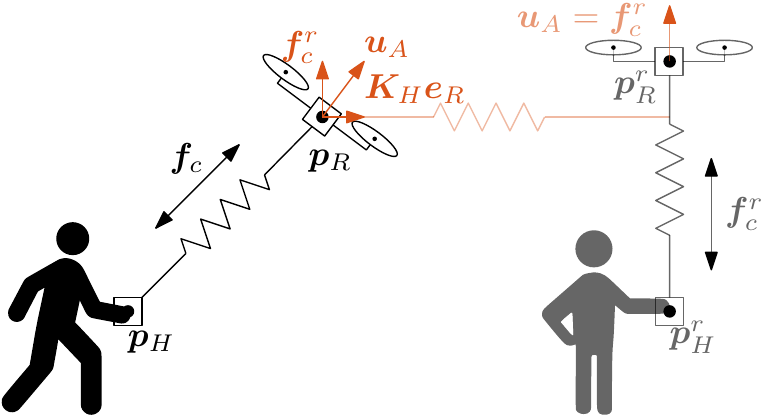}
	\caption{Working principle of the control method. On the left the current state, on the right in opaque, the final desired state.}
	\label{fig:controller:workingPrinciple}
\end{figure}
\begin{figure}[t]
	\centering
	\includegraphics[width = \columnwidth]{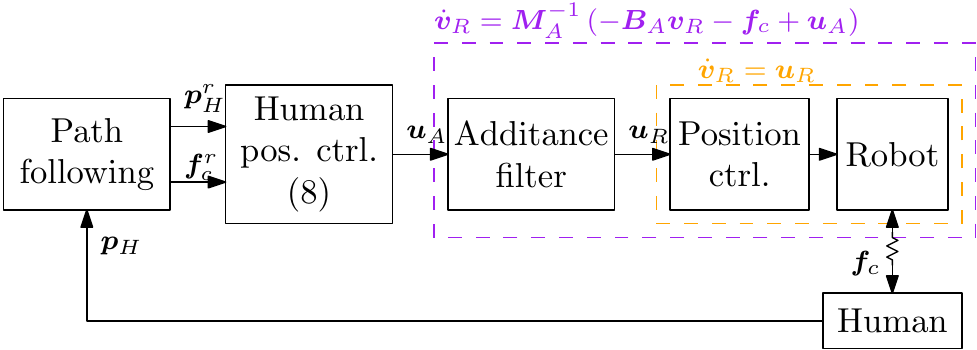}
	\caption{Block diagram of the overall control method.}
	\label{fig:controller:blockDiagram}
\end{figure}
\section{Equilibria and Stability Analysis}\label{sec:equilibriaAndStability}

In the following, we shall show that, by properly choosing $\pRRef$ and $\cableForceRef$, the control law \eqref{eqn:controller} can steer the human position to the desired value. 

\begin{thm}\label{thm:stability}
	Let us define $\pHRef \in \nR{3}$ as the human position reference such that $\zW^\top\pHRef = 0$. Considering the system \eqref{eqn:closedLoopSystemDyn} under the control law  \eqref{eqn:controller} where 
	\begin{align}
		\cableForceRef &= \cableForceZDes\zW \label{eqn:regulation:cableForceRef} \\
		\pRRef &= \pHRef + \cableForceRef \left( \frac{1}{\springCoeff} + \frac{\length}{\norm{\cableForceRef}} \right),\label{eqn:regulation:pRRef}
	\end{align}
	the zero velocity equilibrium $\stateEq = [\pHRef^\top \vSpace \vZero^\top \vSpace \pRRef^\top \vSpace \vZero^\top]^\top \in \stateSetTaut$ is asymptotically stable if $0 < \cableForceZDes < \massH g$.
\end{thm}
\begin{IEEEproof}
	Let us first compute the equilibrium state.
	Plugging \eqref{eqn:controller} into \eqref{eqn:closedLoopSystemDyn} yields the final closed-loop dynamics:
\begin{align}
	\dynamics{\state} &= \matrice{\dpH \\  \frac{1}{\massH}\left(-\gravityH - \dampingH\dpH + \cableForce +  \groundForce\right) \\ \dpR \\ \inertiaA^{-1}\left( -\dampingA\dpR - \cableForce + \KpH \errorR + \cableForceRef \right) }.
	\label{eqn:closedLoopSystemDyn:regulation}
\end{align}
Imposing the stability condition into \eqref{eqn:closedLoopSystemDyn:regulation}, i.e., $\dstate = \dynamics{\state} =\vZero$, from the second row we get:
\begin{align}
	\cableForce = \left( \groundForceIntensity -\massH g \right) \zW.
	\label{eqn:cableForceEq}
\end{align}
This means that at the equilibrium the cable force, and thus the cable itself, have to be vertical.
From the fourth row in \eqref{eqn:closedLoopSystemDyn:regulation} and replacing \eqref{eqn:cableForceEq} we obtain
\begin{align}
	\vZero = \left(\groundForceIntensity  -\massH g + \cableForceZDes \right) \zW + \KpH\errorR. 
	\label{eqn:regulation:equilibriumCondition:1}	
\end{align}
Expanding the components of $\KpH\errorR$ along the axes of $\frameW$, it ies easy to verigy that \eqref{eqn:regulation:equilibriumCondition:1}  is satisfied if:
\begin{align}
	\begin{array}{ccc}
		\xW^\top\errorR = 0 & \yW^\top\errorR = 0 & \groundForceIntensity =  \massH g - \cableForceZDes, 
	\end{array}
\end{align}
which implies: 
\begin{align}
	\begin{array}{ccc}
		\pRX = \pRXRef & \pRY = \pRYRef & \cableForce = \cableForceRef.
	\end{array}
	\label{eqn:regulation:equilibriumCondition:2}
\end{align}
We recall that the intensity of the ground reaction force should be equal or greater than zero, i.e., $\groundForceIntensity \geq 0$. 
Thus, it has to be $\cableForceZDes \leq \massH g$.
Notice that there are no stable configurations for which $\cableForce = \vZero$, i.e., such that $\norm{\cableAttitude} \leq \length$.

To conclude the computation of the equilibrium configuration, we need to calculate the position of the human, $\pH$, and the vertical position of the robot, $\pRZ$.
Replacing conditions \eqref{eqn:regulation:equilibriumCondition:2} into \eqref{eqn:cableForce}, after a few simple calculations, we obtain:
\begin{align}
	\matrice{\pRXRef \\ \pRYRef \\ \pRZ} + \matrice{\pHX \\ \pHY \\ 0} = \matrice{0 \\ 0 \\ \frac{\cableForceZDes}{\springCoeff} + \length}.
\end{align}
Replacing also \eqref{eqn:regulation:cableForceRef} and \eqref{eqn:regulation:pRRef}, we finally obtain that, at the equilibrium, $\pH = \pHRef$ and $\pR = \pRRef$.
Thus, the only equilibrium state for the system \eqref{eqn:closedLoopSystemDyn} is $\stateEq = [\pHRef^\top \vSpace \vZero^\top \vSpace \pRRef^\top \vSpace \vZero^\top]^\top \in \stateSetTaut$.

Now we analyze the stability of the equilibrium state $\stateEq$ employing a Lyapunov-based analysis.
Let us define the following Lyapunov function:
\begin{align}
	\lyapunovFunction = V_1(\state) + V_2(\state) + V_0, 
	\label{eqn:lyapunovFunction:regulation}
\end{align}
where,
\begin{align}
	V_1(\state) &=  \frac{1}{2}\left( \massH\dpH^\top\dpH +  \dpR^\top \inertiaA \dpR  + \errorR^\top \KpH \errorR \right) \\
	V_2(\state) &= \int_{\cableAttitudeNormInitial}^{\cableAttitudeNormOf{t}} \tensionOf{\tau} d\tau   -\cableAttitude^\top\cableForceEq \\
	V_0 &= \cableForceZDes^2/\springCoeff + 2 \length \cableForceZDes.
\end{align}
We defined $\cableAttitudeNormInitial = \cableAttitudeNormOf{0}$.	
In the following we shall show that the defined Lyapunov function, $\lyapunovFunction$, is continuously differentiable and radially unbounded. 

 	$\lyapunovFunction$ is continuously differentiable because $V_1$, $V_2$, $V_0$ are clearly continuously differentiable. 

We then show that \eqref{eqn:lyapunovFunction:regulation} is \textit{radially unbounded} (also called \textit{coercive}), i.e., $\lim_{\norm{\state}\to\infty} V(\state) = \infty$ with $\state\in\stateSet$.
We have that clearly $\lim_{\norm{\state}\to\infty} {V}_1(\state) = \infty$ for $\state \in \stateSet$. For ${V}_2$, we firstly notice that, thanks to \eqref{eqn:tensionMajority}, there exist some $\epsilon \in \nR{}_{> 0}$ and $\gamma \in \nR{}$ such that:
\begin{align}
	\int_{\cableAttitudeNormInitial}^{\cableAttitudeNormOf{t}} \tensionOf{\tau} d\tau \geq  \frac{\epsilon}{2} \cableAttitudeNorm^2 + \gamma\cableAttitudeNorm.
\end{align}
This implies that:
\begin{align}
	\begin{split}
	&\lim_{\norm{\state}\to\infty} V_2(\state) \geq 
	\lim_{\norm{\state}\to\infty} \epsilon \cableAttitudeNorm^2 + \gamma\cableAttitudeNorm - \norm{\cableAttitude}\norm{\cableForceRef} =   \\
	&
	\lim_{\norm{\state}\to\infty}  \norm{\cableAttitude}^2 \left( \frac{\epsilon}{2} + \frac{\gamma}{\norm{\cableAttitude}} - \frac{\norm{\cableForceRef}}{\norm{\cableAttitude}}\right) = +\infty.
	\end{split}\label{coercive}
\end{align}
This is enough to show that  \eqref{eqn:lyapunovFunction:regulation} is {radially unbounded}.

In order to prove that $\lyapunovFunction$ is positive definite, we show that it has a unique global minimum in $\stateEq$, i.e., $\stateEq = \argmin_{\state} V(\state)$ and $V(\stateEq) = 0$.
Since $\lyapunovFunction$ is continuous and radially unbounded, for Theorem 1.15 of \cite{2000-HorParVan}, we can say that the Lyapunov function \eqref{eqn:lyapunovFunction:regulation} has a global minimum.
Now we can look for this minimum among the stationary points, i.e., the $\state$ where the gradient $\nabla V(\state) = \vZero$.

It is clear that $\nabla V_1(\state) = \vZero$ only if $\dpH = \dpR = \vZero$, $\pRX = \pRXRef$ and $\pRY = \pRYRef$. 
Regarding $V_2(\state)$, let us consider its gradient w.r.t. the cable configuration $\cableAttitude$. For every $\state \in \stateSetSlack$:
\begin{align}
	\nabla_{\cableAttitude} V_2(\state) = \frac{\partial  V_2(\state)}{\partial \cableAttitude}  = \frac{\partial }{\partial \cableAttitude} \left( \int_{\cableAttitudeNormInitial}^{\cableAttitudeNorm} \tensionOf{\tau} d\tau \right)- {\cableForceRef}^\top.
	\label{eqn:partialDerivativeV2:partial}
\end{align}
Let us notice that for the Leibniz integral rule we have that:
\begin{align}
	\begin{split}
	\frac{\partial }{\partial \cableAttitude} \left( \int_{\cableAttitudeNormInitial}^{\cableAttitudeNorm} \tensionOf{\tau} d\tau \right) &= \tension \frac{\partial \cableAttitudeNorm}{\partial \cableAttitude} - \tensionOf{\cableAttitudeNormInitial} \frac{\partial \cableAttitudeNormInitial}{\partial \cableAttitude} \\
	&= \tension \frac{\cableAttitude^\top}{\cableAttitudeNorm} = \cableForce^\top,
	\end{split}
	\label{eqn:partialDerivativeTensionIntegral}
\end{align}
where we used \eqref{eqn:cableForce}. 
Plugging \eqref{eqn:partialDerivativeTensionIntegral} into \eqref{eqn:partialDerivativeV2:partial} yields $\nabla_{\cableAttitude} V_2(\state) = \cableForce^\top - {\cableForceRef}^\top$.
It follows that $\nabla_{\cableAttitude} V_2(\state) = \vZero$ if and only if 
\begin{align}\label{lyap_diff_equal0}
\tension{\cableAttitude^\top}/{\norm{\cableAttitude}} = \cableForce^\top = \cableForceRef^\top.
\end{align} 
Condition \eqref{lyap_diff_equal0} holds only if
\begin{enumerate}[a)]
	\item \label{a} $\tension = \norm{\cableForceRef}$ and $\cableAttitude/\cableAttitudeNorm = \cableForceRef / \norm{\cableForceRef}$; or
	\item \label{b} $\tension = -\norm{\cableForceRef}$ and $\cableAttitude/\cableAttitudeNorm = -\cableForceRef / \norm{\cableForceRef}$.
\end{enumerate} 
In these two cases, the cable produces the same force $\cableForceRef$, but in case \ref{a}) the cable is stretched, while in case \ref{b}) the cable is compressed.
However, cables cannot be compressed and thus case \ref{b}) is not feasible for definition, i.e., $\tension \geq 0$ for every $\cableAttitude$.
Therefore, the only solution is:  
\begin{align}
	\cableAttitude = \left(\frac{1}{\springCoeff} + \frac{\length}{\norm{\cableForceRef}} \right) \cableForceRef,
\end{align}
which, recalling \eqref{eqn:regulation:pRRef}, leads to the conclusion that $\nabla_{\cableAttitude} V_2(\state) = \vZero$ only if $\pH = \pHRef$ and $\pR = \pRRef$. 
This completes the proof that $\nabla V(\state) = \vZero$ only if $\state = \stateEq$, and thus $\stateEq = \argmin_{\state} V(\state)$.
Finally, it is easy to verify that $V_1(\stateEq) = 0$ and $V_0$ has been defined such that $V_1(\stateEq) + V_2(\stateEq) = - V_0$, in order to make $V(\stateEq) = 0$. 
This proves that $\lyapunovFunction$ is positive definite.

We can now compute the time derivative of \eqref{eqn:lyapunovFunction:regulation}.
Similarly to \eqref{eqn:partialDerivativeTensionIntegral}, applying the Leibniz integral rule, we have that:
\begin{align}
	\begin{split}
	&
	\frac{d}{dt} \left( \int_{\cableAttitudeNormInitial}^{\cableAttitudeNormOf{t}} \tensionOf{\tau} d\tau \right) =  \tensionOf{\cableAttitudeNorm} \frac{\partial \cableAttitudeNorm}{\partial \cableAttitude} \dCableAttitude = \cableForce^\top\dCableAttitude.
	\end{split}
	\label{eqn:timeDerivativeTensionIntegral}
\end{align}
Therefore, the time derivative of \eqref{eqn:lyapunovFunction:regulation} is equal to:
\begin{align}
	\begin{split}
	\dLyapunovFunction =& \dpH^\top \left( -\gravityH - \dampingH\dpH + \cableForce + \groundForce \right) \\
		& + \dpR^\top \left(  -\dampingA\dpR - \cableForce + \KpH \errorR + \cableForceRef \right) \\
		& - \dpR \KpH \errorR + \dCableAttitude\cableForce - \dCableAttitude\cableForceRef.
	\end{split}  
\end{align}
After few simple algebraic steps we get:
\begin{align}
	\dLyapunovFunction = - \dpH^\top\dampingH\dpH - \dpR^\top\dampingA\dpR.
\end{align}
$\dLyapunovFunction$ is clearly negative semidefinite, in $\stateSet$.

Since $\dLyapunovFunction$ is only negative semidefinite, to prove the asymptotic stability we rely on \textit{LaSalle's invariance principle}~\cite{2002-Kha}. 
	The Lyapunov candidate $\lyapunovFunction$ is a continuously differentiable function.
	Let us define a positively invariant set $\invariantSet = \{\state \in \stateSet \; | \; V(\state) \leq \alpha \text{ with } \alpha\in\nR{}_{>0}\}$.
	By construction, $\invariantSet$ is compact since~\eqref{eqn:lyapunovFunction:regulation} is radially unbounded. Notice that $\invariantSetZero$ contains only $\stateEq$.
	Then we need to find the largest invariant set $\maxInvariantSet$ in $\dVZeroSet = \{ \state \in \invariantSet \; | \; \dLyapunovFunction = 0\}$. 
	A trajectory $\state(t)$ belongs identically to $\dVZeroSet$ if $\dot{V}(\state(t)) \equiv 0 \Leftrightarrow 
	\dpH(t) = \dpR(t) \equiv \vZero \Leftrightarrow
	\ddpH = \ddpR = \vZero$. 
	As we saw previously in the calculation of the equilibria, this is verified only if $\state = \stateEq$.
	Therefore $\maxInvariantSet$ contains only $\stateEq$.
	All conditions of LaSalle's principle are satisfied and $\stateEq$ is globally asymptotically stable.

\end{IEEEproof}

To get more hints on the behavior of the system let us analyze it
 for $\state \in \stateSetSlack$, namely when the cable is slack.
	As we said before, human's and robot's dynamics are independent. 
	The human dynamics results to be:
	\begin{align}
		\massH\ddpH &= -\dampingH\dpH. 
	\end{align}
	It is clear that whatever is the initial condition $\lim_{t \rightarrow +\infty} \dpH(t) = \vZero$. 
	Notice that the human will stop but his final position depends only on the initial conditions.
	
	For the robot dynamics, let us write first the dynamic equation along $\xW$:
	\begin{align}
		m_A\ddpRX &= -b_A\dpRX + \kpH \errorRX, 
	\end{align}
	where $m_A \in \nR{}_{> 0}$ and $b_A  \in \nR{}_{> 0}$ are such that $\inertiaA = \diag{m_A,m_A,m_A}$ and  $\dampingA = \diag{b_A,b_A,b_A}$.
	It is easy to verify that the state $(\dpRX,\errorRX) = (0,0)$ is asymptotically stable.
	The same for the robot dynamics along $\yW$: the state $(\dpRY,\errorRY) = (0,0)$ is asymptotically stable.
	On the other hand, the robot dynamics along $\zW$ results to be:
	\begin{align}
		m_A\ddpRZ &= -b_A\dpRZ + \cableForceZDes.		
	\end{align}
	It is clear that, whatever is the initial condition, $\lim_{t \rightarrow +\infty} \dpRZ(t) = \cableForceZDes / b_A > 0$.
	This means that the robot horizontal position will tend to the desired one, while, along the vertical axes, it will tend to fly up until the cable will become taut.
	In fact, $\frac{d}{dt}\norm{\cableAttitude(t)} = \norm{\dpR(t) - \dpH(t)} \rightarrow \cableForceZDes / b_A > 0$. This means that $\norm{\cableAttitude}$ will grow until $\norm{\cableAttitude}>\length$, i.e., until the cable becomes taut. 
	After this instant, the robot interacts with the human through the cable and eventually pulls him/her to the desired position.
	According to the specific initial condition, it might be that the state trajectory might pass from $\stateSetTaut$ to $\stateSetSlack$ several times. 
	However, the previous Lyapunov-based reasoning proofs that eventually the system will converge to $\stateEq$, no matter the initial conditions.

\section{Robustness of the system}\label{sec:passivity}

In the previous section, we formally proved that an admittance strategy together with controller \eqref{eqn:controller} asymptotically steer the human position to a desired constant position $\pHRef$.
This is valid if the human follows the force provided by the robot without exerting other forces except for the damping effect.
In the following, we shall show the intrinsic robustness of the controller against time-varying position references and possible additional human forces (e.g., if he/she wants to stop).

 \subsection{Time-varying position reference}\label{sec:passivity:trajectory}
We remark that with the control law \eqref{eqn:controller}, the resulting final path followed by the human to arrive at the goal only depends on the initial conditions. 
In a real application, this is not a desired behavior, e.g., when a blind human must be brought to a specific point in an environment with possible obstacles.
In this scenario, it is more likely that the human position reference is not constantly equal to the final goal position, but is rather a time-varying path $\pHRef(t)$ that leads the human to the goal, avoiding obstacles or possible  sources of danger.
To face this objective, taking into account a possibly time-varying path $\pHRef(t)$, and time-varying desired cable force $\cableForceRef(t)$, for $t\in [0,T]$, we modify the original control law \eqref{eqn:controller} as:
\begin{align}
	\uA(t) = \KpH \errorRFinal(t) + \cableForceRefFinal + \timeVaryingInput(t),
	\label{eqn:controller:timeVaryingReference}
\end{align}
where $\errorRFinal(t) = \pRRef(T) - \pR(t)$, $\cableForceRefFinal = \cableForceRef(T) = \cableForceZDes\zW$, and $\timeVaryingInput(t)$ is such that $\uA(t) = \KpH (\pRRef(t) - \pR(t)) + \cableForceRef(t)$ for every $t\in [0,T]$. 
This is a simple way to gather in $\timeVaryingInput$ the time-varying desired quantities.
It results that $\timeVaryingInput(t) =  \KpH( \pRRef(t) - \pRRef(T)) + (\cableForceRef(t) - \cableForceRef(T))$.

To ensure that the system remains stable even when the human position reference is time-varying, i.e., for  $\timeVaryingInput(t) \neq \vZero$, we shall prove that the system is output-strictly passive w.r.t. the input-output pair $(\passiveInputTrajectory,\passiveOutputTrajectory)=(\timeVaryingInput,\dpR)$.

\begin{thm} \label{thm:passivity:trajectory}
If $\uA$ is defined as in \eqref{eqn:controller:timeVaryingReference} for a certain $\pHRef(t)$ and $\cableForceRef(t)$, with $t \in [0,T]$, then system \eqref{eqn:closedLoopSystemDyn} is output-strictly passive w.r.t.
the storage function \eqref{eqn:lyapunovFunction:regulation} and the input-output pair $(\passiveInputTrajectory,\passiveOutputTrajectory)=(\timeVaryingInput,\dpR)$.
\end{thm}
\begin{IEEEproof}
In the proof of Theorem~\ref{thm:stability}  we already showed that \eqref{eqn:lyapunovFunction:regulation} is a continuously differentiable positive definite function. 
Replacing \eqref{eqn:controller:timeVaryingReference}
into \eqref{eqn:closedLoopSystemDyn}, and differentiating \eqref{eqn:lyapunovFunction:regulation} we obtain:
\begin{align}
	\dLyapunovFunction &= - \dpH^\top\dampingH\dpH - \dpR^\top\dampingA\dpR + \timeVaryingInput^\top\dpR \\
		& \geq - \dpR^\top\dampingA\dpR + \timeVaryingInput^\top\dpR = \passiveInputTrajectory^\top\passiveOutputTrajectory - \passiveOutputTrajectory^\top\vect{\phi}_1(\passiveOutputTrajectory),
\end{align}
with $\vect{\phi}_1(\passiveOutputTrajectory) = \dampingA\dpR$. Since $\passiveOutputTrajectory^\top\vect{\phi}_1(\passiveOutputTrajectory) \geq 0$ for every $\passiveOutputTrajectory \neq \vZero$, we can conclude that the system \eqref{eqn:closedLoopSystemDyn}  is \textit{output-strictly passive}~\cite{2002-Kha}.
\end{IEEEproof}

Given the passivity of the system, we have that for a bounded input $\timeVaryingInput$, i.e., bounded reference trajectories, the overall energy
of the system remains bounded too. 
Furthermore, the system stabilizes to a constant value as soon as $\timeVaryingInput$ becomes constant again.
This means that, while the robot moves following $\pRRef(t)$ and $\cableForceRef(t)$, the overall state of the system will remain bounded.
It will then converge to another specific equilibrium configuration when
the robot input $\timeVaryingInput$ becomes constant. 

We highlight the fact that passivity is a well known robust property, especially w.r.t. model uncertainties. 
This means that an accurate model of the system is not necessary, as far as the passivity is guaranteed, to ensure the stability of the system.
For example, considering a certain parameter uncertainty like the stiffness of the cable or its rest-length, the system remains asymptotically stable but it might converge to a slightly different state. 
In particular, the robot will be at a different altitude, which is not a critical problem in practice.

\subsection{Additional human force}\label{sec:humanForce}
 
In \sect\ref{sec:modeling} we modeled the human as a simple mass-damper system subjected to the force exerted by the cable only. 
Nevertheless, during the guidance, the human may want to stop or to deviate from the desired path, e.g., to pick up an object.
In these cases, the human will generate additional forces that will modify his/her trajectory.
It is important that, even in this condition, the system is guaranteed to be stable and that the state converges to the desired value as soon as the human forces go back to zero.

We define the vector $\humanForce \in \nR{3}$ as the sum of all the forces applied by the human that generate a translational motion. 
The closed-loop system dynamics, expressed in \eqref{eqn:closedLoopSystemDyn}, becomes:
\begin{align}
	\dynamics{\state} &= \matrice{\dpH \\  \frac{1}{\massH}\left(-\gravityH - \dampingH\dpH + \cableForce +  \groundForce + \humanForce \right) \\ \dpR \\ \inertiaA^{-1}\left( -\dampingA\dpR - \cableForce + \uA \right) }.
	\label{eqn:closedLoopSystemDyn:regulation:humanForces}
\end{align}

To ensure that the system remains stable even when the human forces are not zero, i.e., for  $\humanForce \neq \vZero$, we shall prove that the system is output-strictly passive w.r.t. the input-output pair $(\passiveInputHumanForce,\passiveOutputHumanForce)=(\humanForce,\dpH)$.

\begin{thm} \label{thm:passivity:humanForce}
System \eqref{eqn:closedLoopSystemDyn:regulation:humanForces} under the control law \eqref{eqn:controller} is output-strictly passive w.r.t.
the storage function \eqref{eqn:lyapunovFunction:regulation} and the input-output pair $(\passiveInputHumanForce,\passiveOutputHumanForce)=(\humanForce,\dpH)$.
\end{thm}
\begin{IEEEproof}
In the proof of Theorem~\ref{thm:stability}  we already showed that \eqref{eqn:lyapunovFunction:regulation} is a continuously differentiable positive definite function. 
Considering \eqref{eqn:closedLoopSystemDyn:regulation:humanForces} and differentiating \eqref{eqn:lyapunovFunction:regulation} we obtain:
\begin{align}
	\dLyapunovFunction &= - \dpH^\top\dampingH\dpH - \dpR^\top\dampingA\dpR + \humanForce^\top\dpH \\
		& \geq - \dpH^\top\dampingH\dpH + \humanForce^\top\dpH  = \passiveInputHumanForce^\top\passiveOutputHumanForce - \passiveOutputHumanForce^\top\vect{\phi}_2(\passiveOutputHumanForce),
\end{align}
with $\vect{\phi}_2(\passiveOutputHumanForce) = \dampingH\dpH$. Since $\passiveOutputHumanForce^\top\vect{\phi}_2(\passiveOutputHumanForce) \geq 0$ for every $\passiveOutputHumanForce \neq \vZero$, we can conclude that the system \eqref{eqn:closedLoopSystemDyn:regulation:humanForces}  is \textit{output-strictly passive}~\cite{2002-Kha}.
\end{IEEEproof}

Conclusions similar to the ones in \sect\ref{sec:passivity:trajectory} can be drawn.
In particular, we have that for a bounded input $\humanForce$, i.e., bounded human forces, the overall energy of the system remains bounded too. 
Furthermore, the system stabilizes to a constant value as soon as $\humanForce$ goes back to zero or balance the cable force.
This means that while the human moves subjected to a non-zero (but bounded) human force $\humanForce \neq \vZero$, the overall state of the system will remain bounded.
It will then converge to another specific equilibrium configuration as soon as $\humanForce$ goes back to zero or balance the cable force.

In particular, considering the case in which the human wants to stop in a certain position $\pH$, it has to be that $\dpH = \ddpH = \vZero$. This means that $\humanForce = \cableForce$.
It is then easy to verify that the robot's position at the new equilibrium will be
\begin{align}
	\pR = \left( \springCoeff\eye{3} + \KpH \right)^{-1} \left( \springCoeff\pH  - \KpH\pRRef + \cableForceRef \right). 
\end{align}
Using a Lyapunov-based argument similar to the one in the proof of \theo\ref{thm:stability}, it is easy to show that such zero-velocity equilibrium state is asymptotically stable. 
This provides a high level of robustness to different human behaviors.

Following the same reasoning of Theorem~\ref{thm:passivity:trajectory} and Theorem~\ref{thm:passivity:humanForce}, we can ensure that the system remains stable even when both the human position reference is time-varying and the human forces are not zero, i.e., for $\timeVaryingInput(t) \neq \vZero$ and $\humanForce \neq \vZero$.
\begin{thm} \label{thm:passivity:TrajectoryAndHumanForce}
System \eqref{eqn:closedLoopSystemDyn:regulation:humanForces} under the control law \eqref{eqn:controller:timeVaryingReference} is output-strictly passive w.r.t. the storage function \eqref{eqn:lyapunovFunction:regulation} and the input-output pair $(\passiveInput,\passiveOutput)=([\passiveInputTrajectory^\top \vSpace \passiveInputHumanForce^\top]^\top,[\passiveOutputTrajectory^\top \vSpace \passiveOutputHumanForce^\top]^\top)$.
\end{thm}
\begin{proof}
	The proof is trivial and follows the same steps of the proof of Theorem~\ref{thm:passivity:trajectory} and Theorem~\ref{thm:passivity:humanForce}.
\end{proof}

\begin{rmk}
Of course, stability is guaranteed as far as the robot can provide the input \eqref{eqn:controller:timeVaryingReference}.
In practice, this can be easily ensured defining feasible cable force of reference $\cableForceRef(t)$, and saturating the robot position error $\errorR = \pRRef(t) - \pR$. 
\end{rmk}

\section{Path-following Problem}\label{sec:pathFollowing}

In \sect\ref{sec:passivity:trajectory}, we saw that even if the human position and cable force of reference are time-varying, the system remains stable. 
This can be exploited for the problem of following a desired path. 
In this section we shall define a strategy to lead the human along the desired path, allowing him/her to stop or slightly deviate from the path without causing the instability of the system.

Given a desired human path\footnote{The path is considered given and is such that no obstacles are present in its proximity. 
The minimum distance to the obstacle will be given by the precision of the control method.} defined by its \textit{parametric representation} $\pHDes(\param) \in \nR{3}$, where $\param \in [0,1]$ is called \textit{parameter}, we firstly define the desired cable force along the path $\cableForceDes(\param) \in \nR{3}$.
This is the force that we would like the human feels all along the path.
$\cableForceDes(\param)$ works as a feedforward term, i.e., it is the nominal pulling force that drags the human along the path. 
In standard path-following problems~\cite{2001-DiaJimSevVicCiv}, this is similar to defining the desired velocity along the path. 
Inspired by most of the related works, we define:
\begin{align}
	\cableForceDes(\param) = \cableForceDesXY(\param) + \cableForceDesZ, 
\end{align}
where
\begin{align}
	\cableForceDesXY(\param) = \cableForceDesXYNorm(\param) \frac{\nabla_{\param} \pHDes(\param)}{\norm{\nabla_{\param} \pHDes(\param)}}, \qquad \cableForceDesZ = \cableForceZDes\zW,
\end{align}
and $\cableForceDesXYNorm(\param)$ is such that $\cableForceDesXYNorm(1) = 0$. 
In particular, here we chose a trapezoidal profile that starts with a non-zero value. 
Notice that the planar pulling force, $\cableForceDesXY(\param)$, is tangential to the path and correctly drive the human along the path. 
Furthermore, at the end of the path, i.e., for $\param = 1$, $\cableForceDes(1) = \cableForceDesZ$, which guarantees $\pHDes(1)$ being asymptotically stable if $\pHRef = \pHDes(1)$ and $\cableForceRef = \cableForceDes(1)$, as it is shown in \sect\ref{sec:equilibriaAndStability}.

Now that we have the desired \textit{maneuver} $\eta = \{ (\pHDes(\param),\cableForceDes(\param)) \in \nR{3} \times \nR{3} \;|\; \param \in [0,1] \}$, considering the control strategy described so far (equations \eqref{eqn:admittance} and \eqref{eqn:controller}), we seek a method to choose $\pHRef(t)$ and $\cableForceRef(t)$ such that $\pH(t)$ asymptotically converges to the path $\pHDes(s)$, i.e.,
\begin{align}
	\lim_{t \rightarrow \infty} \norm{\pH(t) - \pHDes(\paramOpt(t))},
\end{align}
for any smooth function $\paramOpt(t): \nR{} \rightarrow [0,1]$~\cite{2002-SkjFosKok}. 
We first set:
\begin{align}
	\pHRef(t) = \pHDes(\paramOpt(t)), \qquad \cableForceRef(t) = \cableForceDes(\paramOpt(t)).
	\label{eqn:maneuverRegu:humanPosCableForce}
\end{align}
We only have to define $\paramOpt(t)$.
A very simple way would be to define a time law such that $\paramOpt(0) = 0$ and $\paramOpt(T) = 1$ for a given time $T \in \nR{}_{>0}$, e.g., using a spline or a linear function.
However, this would fall into the \textit{trajectory tracking} category which suffers from some robustness issues.
Imagine that the human decides to stop or to slightly deviate from the path. At the same time, the desired reference advances increasing the position error.
If not tackled accurately, e.g., with saturations, the increase of the position error might saturate the robot inputs making the system unstable. 
Even if inputs saturation are properly addressed, once the human decides to follow again the robot, the latter will try to quickly catch up with the moving desired state. 
The desired path will not be followed anymore, and possible collisions with obstacles may occur.

To avoid this issue, we implement a \textit{maneuver regulation} approach similar to the one in~\cite{2013n-SpeNotBueFra,1995-HauHin}, defining:
\begin{align}
	\paramOpt(t) = \argmin_{s \in [0,1]} \norm{\pHDes(s) - \pH(t)}.
	\label{eqn:pathProjection}
\end{align}
At every time $t$, the reference human position is the one belonging to $\pHDes$ and closer to the actual human position. $\cableForce$ is computed accordingly.
It is clear that employing this solution, even if the human stops, the position error will not grow. In fact, $\pHRef$ will remain the same.
Notice that to have \eqref{eqn:pathProjection} well defined, $\pHDes$ must be non-intersecting.
Furthermore, it must be that $\cableForceDesXYNorm(s) \neq 0$ for all $s \neq 1$. This is necessary to make the human always advance along the path.

Replacing the maneuver regulation policy \eqref{eqn:maneuverRegu:humanPosCableForce} into the original controller \eqref{eqn:controller}, we obtain:
\begin{align}
	\uA(t) = \KpH (\pRDes(\paramOpt(t) - \pR(t))) + \cableForceDesXY(\paramOpt(t)) + \cableForceDesZ,
	\label{eqn:controller:maneuverRegualtion}
\end{align}
where $\pRDes$ is computed as in  \eqref{eqn:regulation:pRRef} with $\pRDes$ and $\cableForceDes$ instead of $\pRRef$ and $\cableForceRef$, respectively.
The overall control action is divided into three parts:
\begin{enumerate}
	\item the first is a proportional term that tends to bring the robot (and thus the human) to the desired position;
	\item the second is a feedforward force that tends to pull the human along the desired path;
	\item the third is a forcing vertical force that prevents the cable to become slack.
\end{enumerate}
Notice that as soon as $\paramOpt(t) = 1$, \eqref{eqn:controller:maneuverRegualtion} is equal to \eqref{eqn:controller} where $\pHRef = \pHDes(1)$. Thus the asymptotical stability of $\pHDes(1)$ is ensured. The stability of the system for $\param \neq 1$ is also guaranteed thanks to the passivity of the system.

\begin{rmk}
	Notice that to implement \eqref{eqn:pathProjection}, the actual position of the human is required. 
	This can be easily retrieved either with some embedded sensors in the handle (like GPS), or inverting \eqref{eqn:regulation:pRRef} using the actual robot position and cable force.
\end{rmk}

\section{Experimental Validation}\label{sec:experiments}

In the following, we present the results obtained from an experimental campaign apt to validate the proposed method.

\subsection{Testbed setup}

The system employed in these indoor experiments is shown in \fig\ref{fig:model}. It consists of a quadrotor aerial vehicle connected by a cable to a handle hold by a human being.
The latter is wearing all safety protections including glasses, gloves, and a helmet. %
The holder, simply built with a 3D printer, is equipped with motion tracking markers to obtain a measure of $\pH$.
The experiments are performed indoor employing a Motion Capture System (MOCAP) to allow the tracking of the human hand. 
Here, the use of a reliable positioning system has also the advantage to enhance human safety.
Notice that the proposed experiments still encapsulates the main challenges addressed in the work, i.e., the stability of physical human-robot interaction.
The integration of a redundant sensory suit for safe flights outdoors is not the focus of this manuscript and is left as future work.

The cable connecting the human-handle has a rest-length $\length = 1\,[\rm m]$, and a negligible mass (less than $10\,[\rm g]$).
The aerial vehicle is a custom quadrotor platform equipped with a standard flight-controller, four brushless motor controllers (ESCs) regulating the propeller speed in closed-loop~\cite{2017c-FraMal}, and an onboard PC that runs the state estimator, the position controller, and the admittance filter.
The state estimator fuses the data coming from the onboard IMU (accelerometer and gyroscope) at $1\,[\rm kHz]$ and from the MOCAP (position and attitude) at $60\,[\rm Hz]$.
The measure of the cable force required for implementing \eqref{eqn:admittance} is obtained with the model-based wrench observer presented in~\cite{2017e-RylMusPieCatAntCacFra}.
These components run onboard at $1\,[\rm kHz]$ using the software framework based on \textit{TeleKyb}\footnote{The software framework \textit{TeleKyb} is open-source and available at \url{https://git.openrobots.org/projects/telekyb3}}. The admittance filter parameters have been set to $\inertiaA = 0.8 \eye{3}$ and  $\dampingA = 2.4 \eye{3}$, where $\eye{3} = \diag{1,1,1}$.

The controller \eqref{eqn:controller} and the path-following strategy described in \sect\ref{sec:pathFollowing} are instead implemented in Matlab-Simulink.
The input $\uA$ is computed based on the references $\pHRef$ and $\cableForceRef$ calculated as in \eqref{eqn:maneuverRegu:humanPosCableForce}, and sent to the onboard admittance controller at $100\,[\rm Hz]$ via a wifi connection. 
The gain $\kpH$ has been set equal to $4.5$. 
This, together with the gains of the admittance filter, has been set after a tuning campaign aiming at optimizing the user experience.
We wanted the user to feel a reasonable force, but still not too strong, and as smooth as possible.
This has been also obtained properly choosing the desired force reference along the path. 
We set $\cableForceZDes = 1\;[N]$ to keep the cable always taut, and $\cableForceDesXYNorm(\param)$ as a trapezoidal profile such that $\max\left( \cableForceDesXYNorm(\param)\right) = 1.5 \;[N]$ and $\cableForceDesXYNorm(1) = 0$. 
Finally, $\pHDes(s)$ is a smooth path going from $\pHDes(0) = [-2 \vSpace -0.5 \vSpace 0]^\top$ to $\pHDes(1) = [2 \vSpace 0 \vSpace 0]^\top$, avoiding a virtual obstacle placed in the middle of the two points (see \fig\ref{fig:experiment:phases}). 
The trajectory has been computed using a simple piece-wise polynomial function.

The experiment is divided into three phases (see \fig\ref{fig:experiment:phases}):
\begin{enumerate}[Ph.1]
	\item\label{phaseOne} The human stays in a safe place. The robot takes off and goes into an initial point in position-mode. After, the admittance filter is activated with $\uA = \vZero$. This allows the human safely grabbing the handle;
	\item Once the human holds the handle, the proposed guiding control and path-following strategy are activated and the robot pulls the human along the desired path $\pHDes(s)$;
	\item Once the human reaches the final position, the robot will be vertically placed on top of the handle. In this condition, the human does not feel any horizontal pulling force and can release the handle.  The human guiding controller and the admittance filter are disabled. 
	Finally, the robot lands in position-mode.   
\end{enumerate}

\begin{figure}[t]
	\centering
	\includegraphics[width = \columnwidth]{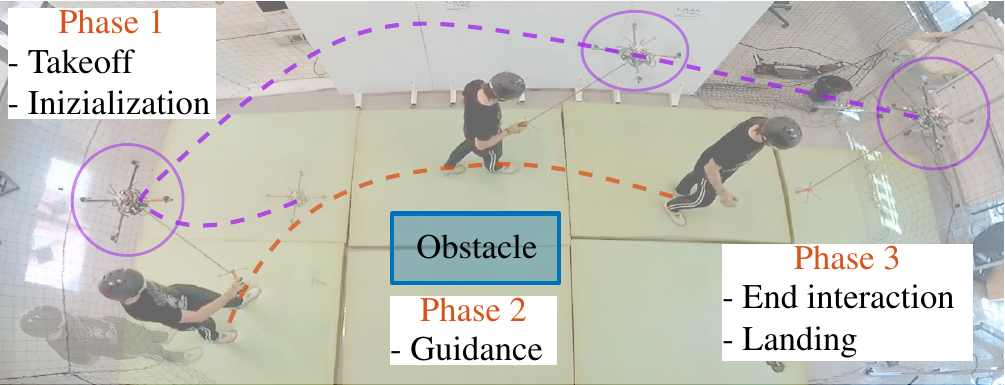}
	\caption{Representation of the experimental phases. On each snapshot the robot is highlighted with a purple circle. Robot and human trajectories are marked with purple and blue dashed lines, respectively.}
	\label{fig:experiment:phases}
\end{figure}

During the experiment, the human is not blinded for safety reasons. 
Nevertheless, he/she is asked to look at the ground and not at the robot, focusing on the forces felt on the hand. 
This is done to encourage the human to be led feeling the forces that the robot is applying through the cable, instead of using it as a visual reference. 
Furthermore, the human is firstly ``trained'' on a different path to make him/her comfortable with the experiment, and confident with the reliability of the robot. 
\begin{figure}[t]
	\centering
	\begin{subfigure}[c]{\columnwidth}
		\includegraphics[width = \columnwidth]{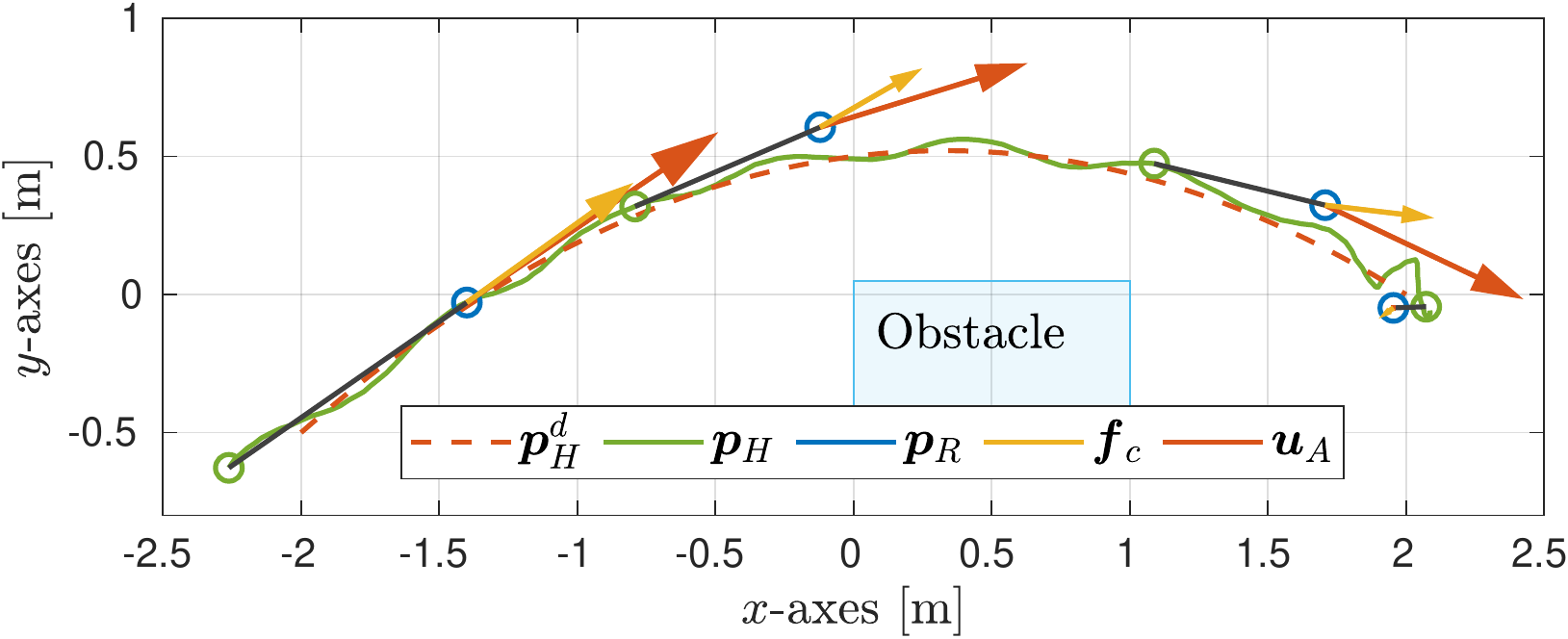}
		\caption{The desired and actual human path showed in the $x$-$y$ plane. Four schematic representation of the system along the path are shown. The green and blue circles represent the human hand and the robot, respectively. The black line represents the cable. The red and yellow arrows represent the $x$-$y$ components of $\uA$ and $\cableForce$, respectively. }
		\label{fig:experiment:single:2D}
	\end{subfigure} \\
	\medskip
	\begin{subfigure}[c]{\columnwidth}
		\includegraphics[width = \columnwidth]{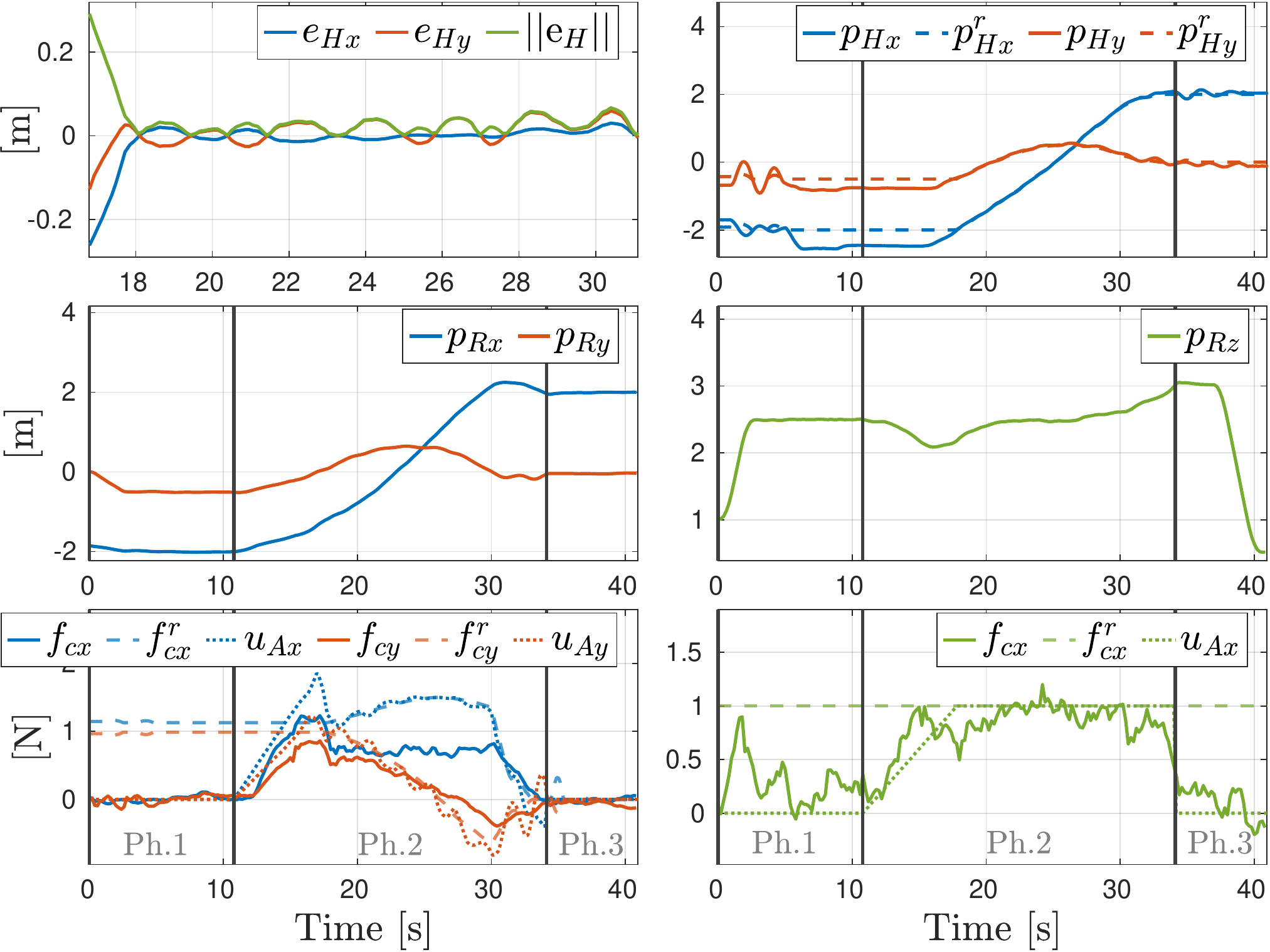}
		\caption{Plots of the human position error (zoomed during the guiding phase), desired and actual human hand position, robot position, and cable force, with respect to time.
		}
		\label{fig:experiment:single:performance}
	\end{subfigure}	
	\caption{Detailed experimental results of one subject.}
	\label{fig:experiment:single}
\end{figure}
\subsection{Experimental results}

In \fig\ref{fig:experiment:single}, we report the detailed results obtained with one participant to the experimental campaign.
In \fig\ref{fig:experiment:single:2D}, we show the desired and actual human path from the top, during {Phase 2}. 
We also schematically depict the system configuration at some instants to show how the robot pulls the human along the desired path.
In \fig\ref{fig:experiment:single:performance}, detailed quantities of the system along the full experiments are shown. 

During \textit{Phase 1}, the robot takes off and goes in the initial desired position. The non-zero estimated force along $\zW$ corresponds to the weight of the handle that in this phase is suspended below the robot. 
The initial oscillations of $\pHx$ and $\pHy$ are also caused by the oscillations of the handle.

Passing to \textit{Phase 2}, the robot starts moving to stretch the cable and pull the human. 
This can be seen from the estimated force along the cable. 
Looking at the human position error $\errorH(t) = \pHRef(t) - \pH(t)$, we remark that its norm is always below $5~[\rm cm]$. 
Although the error is quite small, what is more important is that the stability of the interaction is always guaranteed, even though the human applies non-zero forces. 
The experiment additionally validates the passivity of the system which guarantees stability even in non-ideal conditions. 
One can notice that, during Phase 2, there is an error between the desired horizontal force and the actual one. 
This is because the robot is moving and the damping plays against the generation of the desired force.
Passivity-based methods are known to be very conservative. 
Practically, this is not a problem.
If the human is moving toward the good direction, there is no need to increase the guiding force. 
On the other hand, if the human stops, the robot will apply the desired force, clearly showing the correct direction to the human.

\textit{Phase 3} starts when the human does not feel horizontal forces. 
He/she understands to be arrived at the desired position and so releases the handle. After, the robot lands.

\begin{figure}[t]
	\centering
	\begin{subfigure}[c]{\columnwidth}
		\includegraphics[width = \columnwidth]{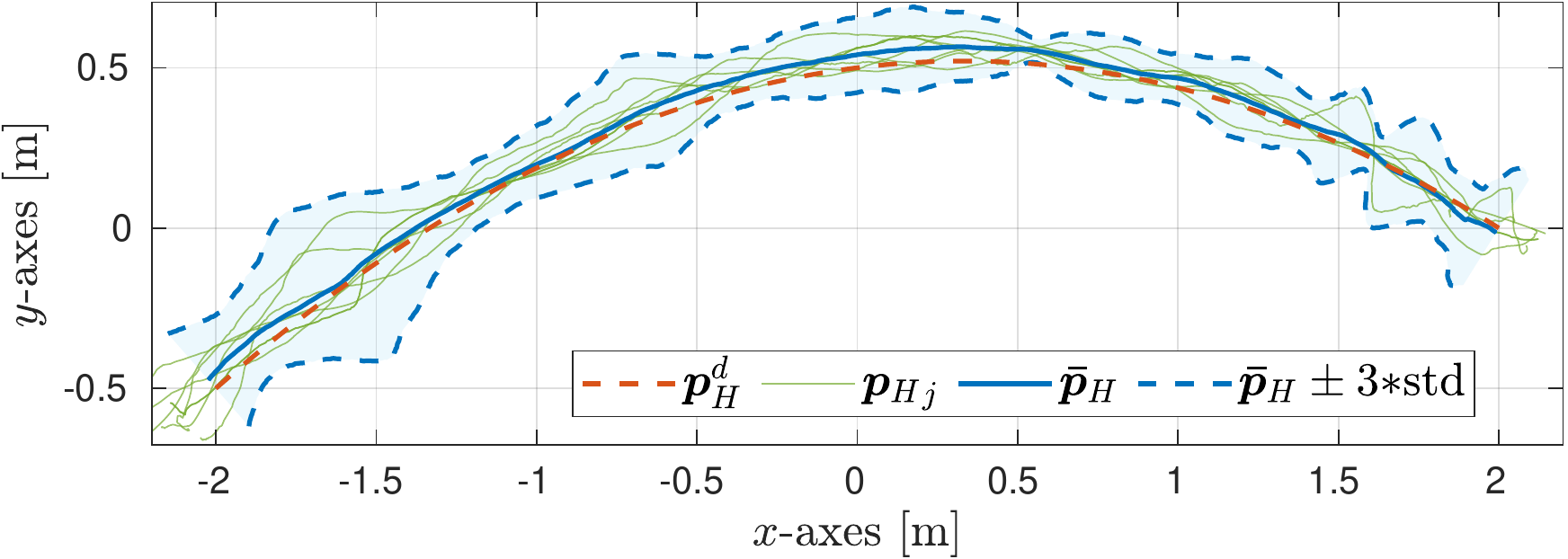}
	\end{subfigure}
	\begin{subfigure}[c]{\columnwidth}
		\includegraphics[width = \columnwidth]{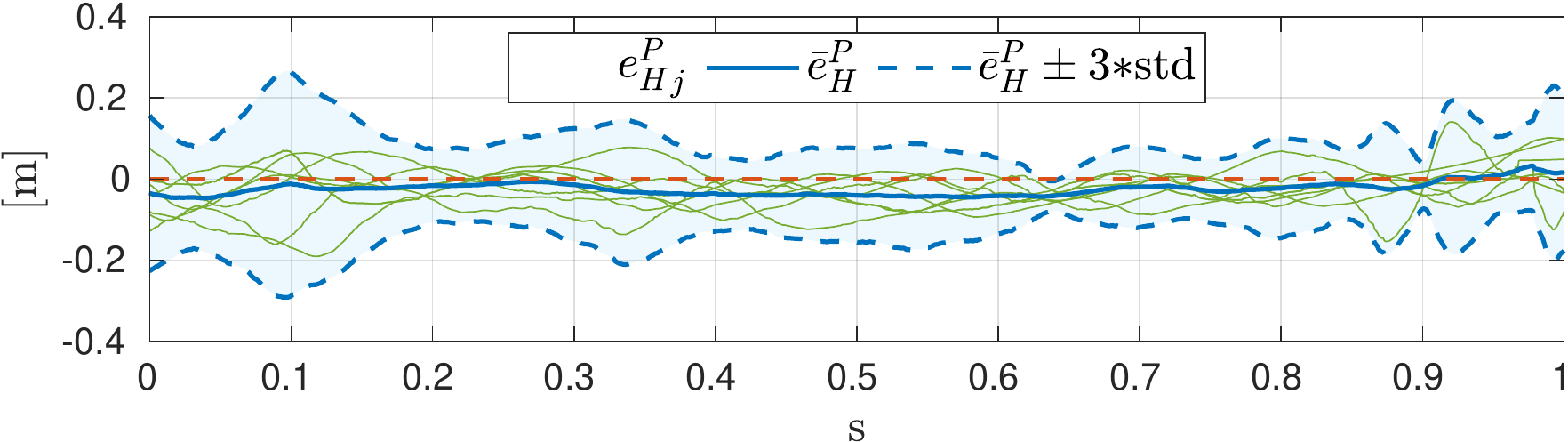}
	\end{subfigure} 
	\caption{Experimental results of eight tests with four subjects. \textit{On the top:} the desired human path, the actual human path for the eight experiments, and the mean human path showed in the $x$-$y$ plane.
	\textit{On the bottom:} the path-following error ${\errorHPath}_j$ for every experiment, and the mean path-following error. }
	\label{fig:experiment:global:2Dtrajectories}
\end{figure}

To further validate the proposed system, we tested it with a group of four different subjects, performing the experiment twice. 
To evaluate the path-following performance, we define the path-following error with respect to the desired path as:
\begin{align}
	\errorHPath(s) = \min_{t \in [0,T]} \left( (\pH(t) - \pHDes(s))^\top \pHDesPerp(s) \right),
	\label{eqn:pathError}
\end{align}
where $\pHDesPerp(s) \in \nR{3}$ is the normal vector to the path in the point $\pHDes(s)$, and $T \in \nR{}_{> 0}$ is the experiment final time.
Given a certain point of the desired path $\pHDes(s)$, parameterized by $\param$, $\errorHPath(s)$ is the distance between $\pHDes(s)$ and the human position intersecting the normal to the path in $\pHDes(s)$. 
Notice that this definition is a \textit{geometric error}; it does not depend on the velocity of the human, neither on his/her trajectory, but just on the geometric path\footnote{Other definitions of the distance between paths can be used. We believe that \eqref{eqn:pathError} allows clearly showing the envelope  around the desired path in which the human will probably pass.}. 
For each $s \in [0,1]$, we then define the mean path-following error, $\errorHPathMean(s)$, and the corresponding standard deviation (std), $\errorHPathStd(s)$:
\begin{align}
	\begin{split}
		\errorHPathMean(s) = \frac{1}{N} \sum_{j=1}^N {\errorHPath}_j(s) \\
		\errorHPathStd(s) = \frac{1}{N -1} \sum_{j=1}^N \left( \errorHPathMean(s) - {\errorHPath}_j(s) \right)^2,
	\end{split}
\end{align}
where $N$ is the number of performed tests, and the subscript $\star_j$ indicates the $j$-th test.
Figure~\ref{fig:experiment:global:2Dtrajectories} shows the global performance.
One can notice that $\errorHPathMean(s)$ is always below $5\,[\rm cm]$, while the maximum standard deviation is below $10\,[\rm cm]$.
This means that, according to the performed experiments, the human will always pass in a tube with diameter $60\,[\rm cm]$ around the desired path. 
This information could be used at the motion planning level to compute paths that stay sufficiently away from obstacles.
We also remark that every subject behaves in a slightly different way (e.g., following the robot with a higher of lower speed, making stops, applying extra forces, etc.). Nevertheless, the system always maintain stability and provide good tracking performance.
We remark that these results do not pretend to be statistically significant. The goal is to provide a global evaluation of the proposed method, which further validates its stability and robustness.

Notice that the human tends to stay always `outward' the desired curved. 
This is the effect of the human inertia that the robot controller does not take into account. 
Furthermore, we noticed that the human might have some difficulties understanding that he/she is arrived, resulting in useless small corrections.
These marginal aspects are left to future works.
\section{Conclusions}\label{sec:conclusions}

In this manuscript, a system composed by a human \textit{physically connected} to an aerial vehicle as been studied for the fist time.
In particular, the human holds a handle which is connected by a cable to an aerial platform. 
The objective is to make the aerial vehicle capable of safely guiding the human to a desired position, or along a desired path, by pulling him/her exploiting the tether.

We proposed a control strategy that makes the robot compliant to the human-interaction and allows to exert forces to the human. 
These forces have been properly designed to drive the human to a desired position.
Using Lyapunov theory, we formally proved that the desired human position is asymptotically stable for the considered closed-loop system.
Furthermore, we showed that the latter is output-strictly passive which guarantees stability and robustness of the method even for time-varying inputs and for non-zero extra human forces.
An additional maneuver regulation strategy has been designed to allow the robot guiding the human along a desired path.

The overall method has been experimentally validated with a group of four subjects.
The results show that the system is able to keep the path-following error bounded. 
The robot correctly drives the human along the desired path, compensating for errors, and guaranteeing the stability.

From a practical point of view, we foresee the integration of onboard sensors for outdoor operations and for the onboard detection of obstacles. 
Based  on this, an online obstacle avoidance policy could be implemented.

From a more scientific point of view, the human model could be enhanced to include the main body and arm elements, together with the heading.
This should be considered in the controller to improve the control of the human  configuration, i.e., position of the main body, heading, and relative position of the arm.
A more precise hybrid position/force controller (like the one in~\cite{2016c-TogDasFra}) could be applied, in order to finely control the intensity and direction of the force felt by the human, regardless to his/her velocity. 

Future works on the HRI aspect can be also done. A human-aware path planner could be designed to explicitly consider at the planning level a path acceptable by the human taking into account human-preferences in terms of distance to obstacles and dynamics of the motion \cite{sisbot2007, Khambhaita2020} and adaptation to human pace \cite{fiore2015}.
The improvement of the indirect force-based communication between human and robot should be investigated as well. 
In particular, impulsive forces or wearable haptic devices could be used to enhance the human-robot interface, e.g., to notice the human that he/she arrived at the desired location. 
On the other hand, the robot should also be able to understand and react to human intentions. 
As an example, if the human wants to stop, the robot should decrease, or even zero, the pulling force.

\section*{Acknowledgments}
\addcontentsline{toc}{section}{Acknowledgment}

We thank Gabriele Minini for his contribution on the early stage of this work.
We are also grateful to Anthony Mallet for his work on the software framework.

\bibliographystyle{IEEEtran}
\bibliography{./bibAlias,./bibCustom}

\begin{IEEEbiography}[{\includegraphics[width=1in,height=1.25in,clip,keepaspectratio]{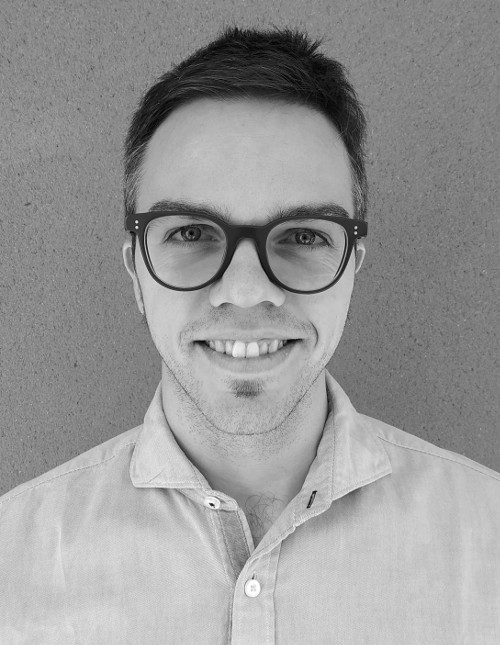}}]{Dr. Marco Tognon} (S'15-M'18) is a postdoctoral researcher at ETH Zurich, Switzerland, in the ASL group.
He received the Ph.D.~degree in Robotics in 2018, from INSA Toulouse, France, developing his thesis at LAAS-CNRS, Toulouse, France.
His thesis has been awarded with three prizes. 
He received the M.Sc. degree in automation engineering in 2014, from the University of Padua, Italy, with a master
thesis carried out at MPI for Biological Cybernetics, T\"ubingen, Germany.
From 2018 to 2020 he has been postdoctoral researcher at LAAS-CNRS, Toulouse, France. 
His current research interests include robotics, aerial physical interaction, and multi-robot systems.
\end{IEEEbiography}

\begin{IEEEbiography}[{\includegraphics[width=1in,height=1.25in,clip,keepaspectratio]{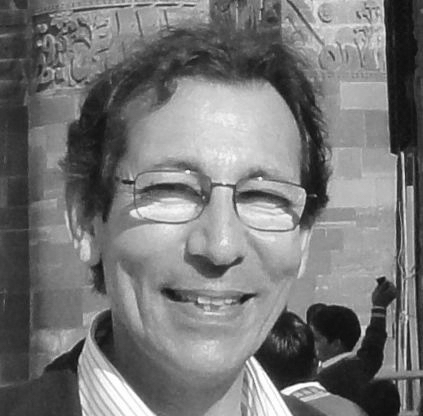}}]{Dr. Rachid Alami} is Senior Scientist at LAAS-CNRS. He received an engineer diploma in computer science in 1978 from ENSEEIHT, a Ph.D in Robotics in 1983 from Institut National Polytechnique and an Habilitation HDR in 1996 from Paul Sabatier University. He contributed and took important responsibilities in several national, European and international research and/or collaborative projects (EUREKA: FAMOS, AMR and I-ARES projects, ESPRIT: MARTHA, PROMotion, ECLA, IST: COMETS, IST FP6 projects: COGNIRON, URUS, PHRIENDS, and FP7 projects: CHRIS, SAPHARI, ARCAS, SPENCER, H2020: MuMMER, France: ARA, VAP-RISP for planetary rovers, PROMIP, several ANR projects).
His main research contributions fall in the fields of Robot Decisional and Control Architectures, Task and motion planning, multi-robot cooperation, and human-robot interaction.
Rachid Alami is currently the head of the Robotics and InteractionS group at LAAS.
He has been offered in 2019 the Academic Chair on Cognitive and Interactive Robotics at the Artificial and Natural Intelligence Toulouse Institute (ANITI)
\end{IEEEbiography}

\begin{IEEEbiography}[{\includegraphics[width=1in,height=1.25in,clip,keepaspectratio]{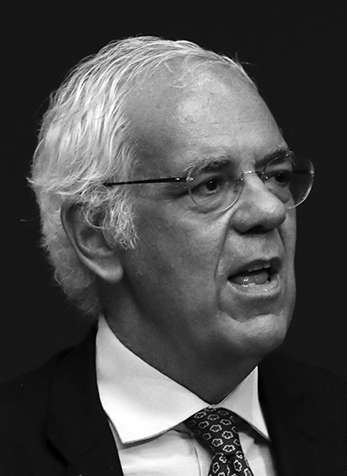}}]{Professor Bruno Siciliano} is Director of the Interdepartmental Center for Advances in RObotic Surgery (ICAROS), as well as Coordinator of the Laboratory of Robotics Projects for Industry, Services and Mechatronics (PRISMA Lab), at University of Naples Federico II. He is also Honorary Professor at Óbuda University, where he holds the Rudolf Kálmán Chair. Fellow of IEEE, ASME, IFAC, he received numerous international prizes and awards, and he was President of the IEEE Robotics and Automation Society from 2008 to 2009. Since 2012 he is on the Board of Directors of the European Robotics Association. He has delivered more than 150 keynotes and has published more than 300 papers and 7 books. His book “Robotics” is among the most adopted academic texts worldwide, while his edited volume “Springer Handbook of Robotics” received the highest recognition for scientific publishing: 2008 PROSE Award for Excellence in Physical Sciences \& Mathematics. His research team got 20 projects funded by the European Union for a total grant of 14 M€ in the last twelve years, including an Advanced Grant from the European Research Council.
\end{IEEEbiography}

\end{document}